\documentclass{article}

\usepackage[utf8]{inputenc} 
\usepackage[T1]{fontenc}    
\usepackage[colorlinks=true,citecolor=blue]{hyperref}       
\usepackage{url}            
\usepackage{booktabs}       
\usepackage{amsfonts}       
\usepackage{nicefrac}       
\usepackage{microtype}      

\usepackage{amsmath,amssymb,amsthm}
\usepackage{xcolor}
\usepackage{colortbl}
\usepackage{times}

\usepackage{graphicx}
\graphicspath{{.//}}

\usepackage{natbib}

\usepackage[mathscr]{eucal}
\usepackage{stmaryrd}
\usepackage{bbold}

\usepackage{mdframed}

\usepackage{float}

\usepackage[boxed]{algorithm}
\usepackage{algorithmic}

\usepackage[small]{caption}

\usepackage{subfig}
\usepackage[shortlabels]{enumitem}

\newcommand{\arxivOnly}[1]{#1}
\newcommand{\nipsOnly}[1]{}

\arxivOnly{
\usepackage[left=1.25in, top=1in, bottom=1in, right=1.25in]{geometry}
}

\newcommand{\mytt}[1]{{\small\begin{alltt}#1\end{alltt}}}

\newcommand{\argmin}{\operatorname{argmin }\, }

\newcommand{\indicator}[1]{\llbracket #1 \rrbracket}
\newcommand{\PMOne}{\{ \pm 1 \}}

\newcommand{\Contaminator}[1]{\bar{#1}}
\newcommand{\etaCont}{\Contaminator{\eta}}
\newcommand{\DCont}{\Contaminator{\D}}
\newcommand{\piCont}{\Contaminator{\pi}}
\newcommand{\PCont}{\Contaminator{P}}
\newcommand{\QCont}{\Contaminator{Q}}

\newcommand{\X}{\mathsf{X}}
\newcommand{\Y}{\mathsf{Y}}
\newcommand{\SSf}{\mathsf{S}}
\newcommand{\SCont}{\Contaminator{\mathsf{S}}}
\newcommand{\YCont}{\Contaminator{\Y}}
\newcommand{\XCal}{\mathscr{X}}
\newcommand{\HCal}{\mathscr{H}}
\newcommand{\SCal}{\mathscr{S}}
\newcommand{\D}{\mathscr{D}}

\newcommand{\RRank}{R_{\mathrm{rank}}}

\renewcommand{\Pr}{\mathbb{P}}
\newcommand{\Expectation}[2]{\mathbb{E}_{#1}\left[ #2 \right]}

\newcommand{\sign}{\mathrm{sign}}
\newcommand{\Real}{\mathbb{R}}

\definecolor{dg}{RGB}{2,101,15}
\newtheoremstyle{dotlessG}{-\topsep}{}{}{}{\color{dg}\bfseries}{.}{ }{}
\definecolor{db}{RGB}{2,15,101}
\newtheoremstyle{dotlessB}{-\topsep}{}{}{}{\color{db}\bfseries}{.}{ }{}
\definecolor{dr}{RGB}{101,2,101}
\newtheoremstyle{dotlessR}{-\topsep}{}{}{}{\color{dr}\bfseries}{.}{ }{}
\definecolor{dy}{RGB}{101,101,2}
\newtheoremstyle{dotlessY}{-\topsep}{}{}{}{\color{dy}\bfseries}{.}{ }{}

\newtheorem{theorem}{Theorem}
\newtheorem{proposition}[theorem]{Proposition}
\newtheorem{lemma}[theorem]{Lemma}
\newtheorem{corollary}[theorem]{Corollary}

\theoremstyle{dotlessR}
\newtheorem*{rem}{Remark}

\theoremstyle{dotlessG}
\newtheorem{defn}{Definition}

\theoremstyle{dotlessB}
\newtheorem{ex}{Example}

\theoremstyle{dotlessY}
\newtheorem{ass}{Assumption}

\newenvironment{definition}
{\begin{mdframed}[innertopmargin=1pt,innerbottommargin=3pt,skipbelow=5pt,backgroundcolor=green!1] \begin{defn}}
{\end{defn} \end{mdframed}}

\newenvironment{example}
{\begin{mdframed}[innertopmargin=1pt,innerbottommargin=3pt,skipbelow=5pt,backgroundcolor=blue!1] \begin{ex}}
{\end{ex} \end{mdframed}}

\newenvironment{remark}
{\begin{mdframed}[innertopmargin=1pt,innerbottommargin=3pt,skipbelow=5pt,backgroundcolor=red!1] \begin{rem}}
{\end{rem} \end{mdframed}}

\newenvironment{assumption}
{\begin{mdframed}[backgroundcolor=yellow!1] \begin{ass}}
{\end{ass} \end{mdframed}}

\newcommand{\SIM}{\mathrm{SIM}}
\newcommand{\GLM}{\mathrm{GLM}}

\newcommand{\SLN}{\mathrm{SLN}}
\newcommand{\CCN}{\mathrm{CCN}}
\newcommand{\PTN}{\mathrm{PTN}}
\newcommand{\BYLN}{\mathrm{BYLN}}
\newcommand{\BCN}{\mathrm{BCN}}
\newcommand{\BCNPlus}{\mathrm{BCN^{+}}}
\newcommand{\IDN}{\mathrm{IDN}}
\newcommand{\ILN}{\mathrm{ILN}}
\newcommand{\SIN}{\mathrm{SIN}}

\newcommand{\MonoLips}{\mathscr{U}}

\newcommand{\weight}{\mathrm{wt}}
\newcommand{\reg}{\mathrm{reg}}
\newcommand{\scorer}{s \colon \XCal \to \Real}

\newcommand{\ellZO}{\ell^{01}}
\newcommand{\ellSQ}{\ell^{\mathrm{sq}}}

\newcommand{\defEq}{\stackrel{\raisebox{-0.25ex}{\scalebox{1.2}{$\cdot$}}}{=}}

\newcommand{\uMar}{u_{\mathrm{mar}( \gamma )}}

\newcommand{\ie}{i.e.\ }
\newcommand{\eg}{e.g.\ }

\allowdisplaybreaks

\title{Learning from Binary Labels with Instance-Dependent Corruption}

\arxivOnly{
\author{
  Aditya Krishna Menon \\
  Data61 and the Australian National University \\
  Canberra, ACT, Australia \\
  {\texttt{aditya.menon@nicta.com.au}} \\
  \and
  Brendan van Rooyen \\
  Queensland University of Technology \\
  Brisbane, QLD, Australia \\
  {\texttt{brendan.vanrooyen@qut.edu.au}} \\
  \and
  Nagarajan Natarajan \\
  Microsoft Research Bangalore \\
  Bengaluru, KT, India \\
  {\texttt{naga86@gmail.com}} \\
}
}
\nipsOnly{
\author{
  Aditya Krishna Menon \\
  Data61 and the ANU \\
  \arxivOnly{Canberra, ACT, Australia \\}
  {\scriptsize\texttt{aditya.menon@nicta.com.au}} \\
  \And
  Brendan van Rooyen \\
  Queensland University of Technology \\
  \arxivOnly{Brisbane, QLD, Australia \\}
  {\scriptsize\texttt{brendan.vanrooyen@qut.edu.au}} \\
  \And
  Nagarajan Natarajan \\
  MSR Bangalore \\
  \arxivOnly{Bengaluru, KT, India \\}
  {\scriptsize\texttt{naga86@gmail.com}} \\
}
}

\begin{document}

\maketitle


\begin{abstract}
Suppose we have a sample of instances paired with binary labels corrupted by arbitrary instance- and label-dependent noise. 
With sufficiently many such samples,
can we optimally classify and rank instances with respect to the noise-free distribution?
We provide a theoretical analysis of this question, with three main contributions.
First, we prove that for instance-dependent noise, any algorithm that is consistent for classification on the noisy distribution is also consistent on the clean distribution.
Second, we prove that for a broad class of instance- and label-dependent noise, a similar consistency result holds for the area under the ROC curve.
Third, for the latter noise model,
when the noise-free class-probability function belongs to the generalised linear model family, 
we show that the Isotron can efficiently and provably learn from the corrupted sample.
\end{abstract}


\section{Learning with label noise: from constant to instance-dependent}

Given an instance space $\XCal$, and training samples from some distribution $\D$ over $\XCal \times \PMOne$, the goal in binary supervised learning is to learn a scorer $s \colon \XCal \to \Real$ with low \emph{risk} on future test samples drawn from $\D$.
Depending on the choice of risk, one arrives at the practically pervasive problems of binary classification \citep{Devroye:1996}, class-probability estimation \citep{Buja:2005}, and bipartite ranking \citep{Agarwal:2005}.
While the standard setup assumes that the train and test distributions are identical,
often the training labels are \emph{corrupted} in some way,
so that the training samples are effectively from some $\DCont \neq \D$.
The case where the labels are flipped with constant or class-dependent probabilities 
have been well-studied of late \citep{Natarajan:2013,Scott:2013,Menon:2015,vanRooyen:2015,Patrini:2016}.

Our interest is the case where training labels are flipped with unknown, \emph{instance- and label-dependent} probabilities.
This challenging setting was recently studied in \citep{Manwani:2013,Ghosh:2015}, who established that certain non-convex losses are robust to such noise, provided the true distribution $\D$ is separable.
However, compared to 
the symmetric- and class-conditional noise case, 
three important questions remain unanswered.
First, is suitable risk minimisation on the corrupted sample \emph{consistent} for minimisation on the clean sample?
Second, does consistency hold for more general risks, such as that for bipartite ranking?
Third, if we have more knowledge as to the structure of $\D$, can we design efficient algorithms to provably learn from the corrupted samples?

In this paper, we provide positive answers to all questions under mild assumptions on the noise process, and an additional assumption on $\D$ for the third question.
Specifically:
\begin{itemize}[leftmargin=0.25in]
  \item on the theoretical side,
we prove that:
  \begin{itemize}[leftmargin=0.1in]
    \item under instance-dependent noise,
the Bayes-optimal scorers for certain losses are unchanged (Corollary \ref{corr:bayes-opt-same}), and that
any algorithm consistent for classification on the noisy distribution is also consistent on the clean distribution (Proposition \ref{prop:regret-bound});
  \item under a broad range of instance- and label-dependent noise,
the corrupted class-probability function preserves the order of the clean one (Proposition \ref{prop:eta-monotone}), and
we have consistency for the area under the ROC curve maximisation on the noisy distribution (Proposition \ref{prop:auc-regret});
  \end{itemize}

  \item on the algorithmic side,
we show that if
$\D$ has class-probability function belonging to the 
generalised linear model family,
then under the aforementioned class of instance- and label-dependent noise, so does the corrupted class-probability function (Proposition \ref{prop:corrupt-sim});
and thus, 
consistent classification and ranking is afforded by the 
Isotron \citep{Kalai:2009} (Proposition \ref{prop:isotron-consistency}).
\end{itemize}

Our analysis relies on the structure of the class-probability function under instance- and label-dependent noise (Lemma \ref{lemm:corrupt-eta-general}).
Our results broadly generalise those for class-conditional label noise in \citet{Natarajan:2013,Menon:2015}, where
this viewpoint has similarly proven useful.

\section{Background and notation}

We fix some notation and introduce some relevant background material.

\subsection{Learning from binary labels}

Fix an instance space $\XCal$.
We denote by $\D$ some distribution over $\XCal \times \PMOne$,
with $(\X, \Y) \sim D$ a pair of random variables. 
Any $\D$ may be expressed via the \emph{class-conditional distributions}
$(P, Q) = ( \Pr( \X \mid \Y = 1 ), \Pr( \X \mid \Y = -1 ) )$ and \emph{base}
rate $\pi = \Pr( \Y = 1 )$, or equivalently
via \emph{marginal distribution} $M = \Pr( \X )$ and \emph{class-probability function} $\eta \colon x \mapsto \Pr( \Y = 1 \mid \X = x )$. 
We assume $\pi \in (0, 1)$.

A \emph{scorer} is any $\scorer$.
A \emph{loss} is any $\ell \colon \PMOne \times \Real \to \Real_+$.
The \emph{$\ell$-risk} of a scorer $s$ wrt $\D$ is 
\begin{align}
R( s; \D, \ell ) &\defEq \Expectation{(\X, \Y) \sim \D}{\ell( \Y, s(\X) )} = \Expectation{\X \sim M}{ L( \eta( \X ), s( \X ) ) },
\end{align}
where $L(\eta, v) \defEq \eta \cdot \ell_1( v ) + (1 - \eta) \cdot \ell_{-1}( v )$ is the \emph{conditional risk} of $\ell$.
The \emph{Bayes-optimal} scorers for a loss $\ell$ are those that minimise the $\ell$ risk.
The \emph{$\ell$-regret} of a scorer $\scorer$ is the excess risk over that of any Bayes-optimal scorer $s^* \in \operatorname{argmin}_s R( s; \D, \ell )$:
\begin{align*}
    \reg( s; \D, \ell ) &\defEq R( s; \D, \ell ) - R( s^*; \D, \ell ) = \Expectation{\X \sim M}{ \reg( \eta( \X ), s( \X ), s^*( \X ) ) }
\end{align*}
where, in an overload of notation, $\reg( \eta, s, s^* ) \defEq L( \eta, s ) - L( \eta, s^* )$.

\subsection{Learning from corrupted binary labels}

\arxivOnly{
In the standard problem of learning from binary labels, we have access to a sample $\SSf = \{ ( x_n, y_n ) \}_{n = 1}^{N} \sim \D^N$.
Our goal is to learn a scorer $s$ from this sample with low $\ell$-risk with respect to $\D$.
}
Fix some notional ``clean'' (not necessarily separable) distribution $\D$.
In the problem of learning from \emph{corrupted} binary labels, we have access to a sample $\SCont = \{ ( x_n, \bar{y}_n ) \}_{n = 1}^{N} \sim \DCont^N$,
for some $\DCont \neq \D$ where $\Pr(\X)$ is unchanged, but $\Pr( \YCont \mid \X = x ) \neq \Pr( \Y \mid \X = x )$.
Our goal remains to learn a scorer $s$ from $\SCont$ with low $\ell$-risk with respect to $\D$.
Examples include learning from symmetric label noise \citep{Angluin:1988}, and class-conditional noise \citep{Blum:1998}.

\arxivOnly{
Note that we allow $\D$ to be non-separable, \ie $\eta( x ) \cdot (1 - \eta( x ) ) > 0$ for some $x \in \XCal$;
thus, even under $\D$, there is not necessarily certainty as to every instance's label.
Our use of ``noise'' and ``corruption'' thus refers to an additional, exogenous uncertainty in the labelling process.
}

\arxivOnly{
\subsection{Existing work on learning from noisy labels}

There is too large a body of work on label noise to fully summarise here (see \eg \citet{Frenay:2014} for a recent survey).
Broadly, there have been three strands of theoretical analysis.
\begin{enumerate}[(1),leftmargin=0.25in]
  \item \emph{PAC guarantees}. The first strand has focussed on PAC-style guarantees for learning under symmetric and class-conditional noise (\eg \citep{Bylander:1994,Blum:1996,Blum:1998}),
noise consistent with the distance to the margin (\eg \citet{Angluin:1988,Bylander:1997,Bylander:1998,Servedio:1999}), 
noise with bounded error rate\footnote{For separable $\D$, this is effectively the agnostic learning problem \citep{Kearns:1994}.} (\eg \citet{Kalai:2005,Awasthi:2014})
and arbitrary bounded instance dependent or Massart noise (\eg \citet{Awasthi:2015}).
These works often assume the true distribution $\D$ is linearly separable with some margin,
the marginal over instances has some structure (\eg uniform over the unit sphere, or log-concave isotropic),
and that one employs linear scorers for learning.

  \item \emph{Surrogate losses}. The second strand has focussed on the design of surrogate losses robust to label noise.
\citet{Long:2008} showed that even under symmetric label noise, convex potential minimisation with such scorers will produce classifiers that are akin to random guessing.
For class-conditional noise, \citet{Natarajan:2013} provided a simple ``noise-corrected'' version of any loss.
\citet{Ghosh:2015} showed that losses whose components sum to a constant are robust to symmetric label noise. 
\citet{vanRooyen:2015} showed that the linear or unhinged loss is robust to symmetric label noise.
\citet{Patrini:2016} showed that a range of ``linear-odd'' losses (LOLs) are approximately robust to asymmetric label noise, provided that the mean operator is not affected too much by corruption.

  \item \emph{Consistency}. The third strand, which is closest to our work, has focussed on 
  showing consistency of appropriate risk minimisation in the regime where one has a suitably powerful function class \citep{Scott:2013,Natarajan:2013,Menon:2015}.
  For example, \citet{Natarajan:2013} showed that minimisation of appropriately weighted convex surrogates on the corrupted distribution $\DCont$ is in fact \emph{consistent} for the purposes of classification on $\D$.
  This work has been restricted to the case of symmetric- and class-conditional noise.
\end{enumerate}

In the present paper, we do not make assumptions on $\D$ for our theoretical analysis (unlike (1)),
assume one is working with a suitably rich function class (unlike (1) and (2)),
and work with general instance- and label-dependent noise models (unlike (2) and (3)).
}



%
\subsection{The SIM family of class-probability functions}

\arxivOnly{
Recall the standard generalised linear model (GLM) family of class-probability functions.
\begin{definition}[GLM]
For any $u \colon \Real \to [ 0, 1 ], w^* \in \Real^d$, the GLM class-probability function is
$$ \GLM( u, w^* ) \defEq x \mapsto u( \langle w^*, x \rangle ). $$
\end{definition}
}
For any $u \colon \Real \to [ 0, 1 ], w^* \in \Real^d$, the generalised linear model (GLM) class-probability function is:
$$ \GLM( u, w^* ) \defEq x \mapsto u( \langle w^*, x \rangle ). $$
Logistic regression \eg assumes $\eta = \GLM(u, w^*)$ for $u( z ) = 1/(1 + e^{-z})$.
We often refer to $u( \cdot )$ as a link function. 
In this paper, we are interested in cases where $u( \cdot )$ is \emph{unknown}, but is known to satisfy some mild properties.
Specifically, we will study the ``single-index model'' (SIM) family of class-probability functions \citep{Kalai:2009}.
(See Appendix \ref{app:sim-family} for some examples.)


\begin{definition}[SIM]
For any $L, W \in \Real_+$, the SIM family of class-probability functions is 
$$ \SIM( L, W ) \defEq \{ \GLM( u, w^* ) \colon u \in \MonoLips( L ), || w^* || \leq W \} $$
where $\MonoLips( L )$ is the set of non-decreasing $L$-Lipschitz functions.
\end{definition}


%
\section{The ILN model of label noise}

We now outline the noise models forming the broad basis of this paper, starting with the most general.

\subsection{An instance- and label-dependent noise model}
\label{sec:iln}

In the general instance- and label-dependent noise model (\emph{ILN model}), a conceptual sample from the true distribution $\D$ has each of its labels flipped with an instance- and label-dependent probability.

\begin{definition}[ILN model]
\label{defn:iln-process}
Let $\rho_1, \rho_{-1} \colon \XCal \to [ 0, 1 ]$.
Given any distribution $\D$, under the ILN model we observe a distribution
$ \ILN( \D, \rho_{-1}, \rho_{1} ) $
whose samples $( \X, \Contaminator{\Y} )$ are generated as follows:
one draws a pair $(\X, \Y) \sim \D$ as usual,
but then flips the label with the \emph{instance- and label-dependent} probability $\rho_{\Y}( \X )$.
\end{definition}

In the sequel, we will always assume the following condition on the flip probability functions.
\nipsOnly{
\begin{flalign}
  \label{ass:total-noise}  
  \textbf{Assumption 1.} && ( \forall x \in \XCal ) \, \rho_1( x ) + \rho_{-1}( x ) < 1. &&
\end{flalign}
}
\arxivOnly{
\begin{assumption}[Bounded total noise]
\label{ass:total-noise}
The label flip functions satisfy
\begin{equation}
    \label{eqn:bounded-noise}
     ( \forall x \in \XCal ) \, \rho_1( x ) + \rho_{-1}( x ) < 1.
\end{equation}
\end{assumption}
}

Assumption \ref{ass:total-noise} simply encodes that 
there is always \emph{some} signal to learn from for each instance.
When the flip functions are constant, the requirement is that $\rho_+ + \rho_- < 1$, a standard condition in analysis of the class-conditional setting (e.g.\ \citet{Blum:1998,Scott:2013}).
To reiterate that the assumption is employed, we will refer to $\rho_{\pm 1}$ satisfying Assumption \ref{ass:total-noise} as being ``admissible''.

%
\subsection{Special cases: the IDN and $\BCNPlus$ model}

There are a few special cases of the general ILN model that will be of interest to us;
see Appendix \ref{app:iln-special-case} for more examples and discussion.
The first one is the well-studied class-conditional noise (CCN) setting, where $\rho_{\pm 1}$ are constants independent of $x$.
The second is where the noise is instance dependent only, which we call the IDN model.

\begin{definition}[IDN model]
Consider an ILN model $ \ILN( \D, \rho_{-1}, \rho_{1} ) $ where
$\rho_{-1} \equiv \rho_1 \equiv f$ for some function $f \colon \XCal \to [0, \nicefrac[]{1}{2})$.
We term this problem learning with instance-dependent noise (IDN learning).
We will write the corresponding corrupted distribution as $\IDN( \D, f )$.
\end{definition}

The third is where, roughly, 
the higher the inherent uncertainty (\ie $\eta \approx \nicefrac[]{1}{2}$), the higher the noise.

\begin{definition}[$\BCNPlus$ model]
\label{defn:bcn-plus}
Consider an ILN model $ \ILN( \D, \rho_{-1}, \rho_{1} ) $
where $\rho_y = f_y \circ s$ for some functions $f_{\pm 1} \colon \Real \to [0, 1]$, and a function $s \colon \XCal \to \Real$ such that:
\begin{enumerate}[(a)]
    \item $s$ is order preserving for $\eta$ \ie
$$ ( \forall x, x' \in \XCal ) \, \eta( x ) < \eta( x' ) \implies s( x ) < s( x' ). $$

    \item $f_{\pm 1}$ are non-decreasing when $\eta \leq \nicefrac[]{1}{2}$,
    and non-increasing when $\eta \geq \nicefrac[]{1}{2}$.

    \item The flip function difference $\Delta( z ) = f_{1}( z ) - f_{-1}( z )$ 
    is non-increasing.
\end{enumerate}

We term this problem learning with generalised boundary consistent noise ($\BCNPlus$ learning).
We will write the corresponding corrupted distribution as $\BCNPlus( \D, f_{-1}, f_{1}, s )$;
further, we will say that $( f_{-1}, f_{1}, s, \eta )$ are \emph{$\BCNPlus$-admissible} if they satisfy the conditions detailed above.
\end{definition}

Condition (a) above implies that $\eta = u \circ s$ for some non-decreasing $u$.
Condition (b) encodes that $f_{\pm 1}$ are highest when $\eta \approx \nicefrac[]{1}{2}$, and lowest when $\eta \cdot (1 - \eta) \approx 0$.
Condition (c) is more opaque, but is trivially satisfied when the flip functions are identical or constant,
and is needed to ensure a monotonicity property of $\etaCont$; this will be discussed in \S\ref{sec:order-preservation}.

A simple example of the $\BCNPlus$ model (studied in \eg \citet{Du:2015}) is when $s( x ) = \langle w^*, x \rangle$ and
$\eta( x ) = \indicator{ s( x ) > 0 }$ \ie $\D$ is linearly separable,
and further $f_{\pm 1}( z ) = g(| z |)$ for some monotone decreasing $g$.
By Condition (b), one has higher noise for instances that are closer to the separator $w^*$.
This is a reasonable model of noise in problems involving human annotation:
the more intriniscally ``hard'' an instance, the higher noise we expect for it.
A similar model was studied in \citet{Bootkrajang:2016} from a probabilistic perspective.

\subsection{The corrupted class-probability function for the ILN model}

%

The nature of the corrupted class-probability function $\etaCont$ for the ILN model (Definition \ref{defn:iln-process}) will serve as the basis for learning from such corrupted samples.

\begin{lemma}
\label{lemm:corrupt-eta-general}
Pick any distribution $\D$.
Suppose $\DCont = \ILN( \D, \rho_{-1}, \rho_{1} )$ for some admissible $\rho_{\pm 1} \colon \XCal \to [ 0, 1 ]$.
Then, $\DCont$ has corrupted class-probability function
\begin{equation}
   \label{eqn:corrupt-eta-general}
   ( \forall x \in \XCal ) \, \etaCont( x ) = (1 - \rho_1( x ) ) \cdot \eta( x ) + \rho_{-1}( x ) \cdot ( 1 - \eta(x) ).
\end{equation}
\arxivOnly{
or equivalently,
$$ ( \forall x \in \XCal ) \, \eta( x ) = \frac{\etaCont( x ) - \rho_{-1}( x )}{1 - \rho_1( x ) - \rho_{-1}( x )}. $$
}
\end{lemma}

\arxivOnly{
\begin{remark}
Were it true that $ \rho_1( x ) + \rho_{-1}( x ) = 1 $ for some $x$, then we would have $\etaCont( x ) = \rho_{-1}( x )$ \ie it is independent of the actual $\eta( x )$ value.    
Thus, Assumption \ref{ass:total-noise} specifies that it is possible to infer something about $\eta( x )$ from $\etaCont( x )$.
\end{remark}
}

Lemma \ref{lemm:corrupt-eta-general} generalises \citet[Lemma 7]{Natarajan:2013}, \citet[Appendix C]{Menon:2015}, who derived $\etaCont$ for the case of CCN learning.
See Appendix \ref{app:additional-properties} for more special cases. 


%
\section{Classification consistency under instance-dependent noise}


%

Suppose one minimises the $\ellZO$-risk on the corrupted distribution.
Does this imply minimisation of the $\ellZO$-risk on the \emph{clean} distribution
\ie is the former 
\emph{consistent} for clean $\ellZO$-risk minimisation?
We will show that this is indeed true for instance-dependent noise, and for a range of losses $\ell$ beyond $\ellZO$.

%
\subsection{Relating clean and corrupted risks}

Our first step is to relate the $\ell$-risk on the clean and corrupted distributions.
Following \citet{Ghosh:2015}, we will consider
instance-dependent noise $\IDN( \D, f )$, and
losses $\ell$ that satisfy
\begin{equation}
    \label{eqn:loss-sum-constant}
    ( \forall v \in \Real ) \, \ell_{-1}( v ) + \ell_{1}( v ) = C
 \end{equation}
for some constant $C \in \Real$.
This condition was considered previously in \citet{Ghosh:2015} to study noise-robustness, and is satisfied by the zero-one, ramp, and unhinged losses.
Under these two assumptions, we can show
the clean risk is an \emph{instance-weighted} version of the corrupted risk. 
To simplify notation, for any $w \colon \XCal \to \Real_+$, let the corresponding \emph{weighted $\ell$-risk} be
$$ R^{\weight( w )}( s; \D, \ell ) = \Expectation{\X \sim M}{ w( \X ) \cdot L( \eta( \X ), s( \X ) ) }. $$
Then, we have the following, which is implicit in the proof of \citet[Theorem 1]{Ghosh:2015}.

\begin{proposition}
\label{prop:risk-weighting}
Pick any distribution $\D$, and
loss $\ell$ satisfying Equation \ref{eqn:loss-sum-constant}.
Suppose that $\DCont = \IDN( \D, f )$
for admissible $f \colon \XCal \to [0, \nicefrac[]{1}{2})$.
Then, for any scorer $\scorer$,
\begin{align*}
    R( s; \D, \ell ) = R^{\weight( w )}( s; \DCont, \ell ) + A( \D, f )
\end{align*}
where
$ w( x ) = ( 1 - 2 \cdot f( x ) )^{-1} $,
and $A( \D, f )$ is some term independent of $s$.
\end{proposition}

%

%
\subsection{Relating clean and corrupted regrets}

Proposition \ref{prop:risk-weighting} has an important, non-obvious implication:
under instance-specific noise, for losses satisfying Equation \ref{eqn:loss-sum-constant},
the Bayes-optimal scorers on the clean and corrupted distributions coincide.
This is a simple consequence of the fact that weighting a risk does \emph{not} affect Bayes-optimal scorers.

\begin{corollary}
\label{corr:bayes-opt-same}
Pick any distribution $\D$, and
loss $\ell$ satisfying Equation \ref{eqn:loss-sum-constant}.
Suppose that $\DCont = \IDN( \D, f )$
for admissible $f \colon \XCal \to [0, \nicefrac[]{1}{2})$.
Then,
$$ \underset{s \in \Real^{\XCal}}{\argmin} R( s; \D, \ell ) = \underset{s \in \Real^{\XCal}}{\argmin} R( s; \DCont, \ell ). $$
\end{corollary}

For the case of $\ell = \ellZO$, Corollary \ref{corr:bayes-opt-same} 
implies that the optimal\footnote{This assumes minimisation over $\Real^{\XCal}$, and may not hold with a restricted function class; see Appendix \ref{app:bayes-opt}.} classifiers on the two distribution coincide: $\sign( 2\eta( x ) - 1 ) = \sign( 2\etaCont( x ) - 1 )$.
In fact, we can go further, 
and establish a relation between the clean and corrupted \emph{regrets} of an arbitrary scorer.

\begin{proposition}
\label{prop:regret-bound}
Pick any distribution $\D$, and
loss $\ell$ satisfying Equation \ref{eqn:loss-sum-constant}.
Suppose that $\DCont = \IDN( \D, f )$
for admissible $f \colon \XCal \to [0, \nicefrac[]{1}{2} )$
with
$ \rho_{\mathrm{max}} = \max_{x \in \XCal} f( x ) $. 
Then, for any $\scorer$, 
$$ \reg( s; \D, \ell ) \leq \frac{1}{1 - 2 \cdot \rho_{\mathrm{max}}} \cdot \reg( s; \DCont, \ell ). $$
Further, if $\sup_{x \in \XCal} \reg( \etaCont( x ), s( x ), s^*( x ) ) \leq R < +\infty$, then for any $\alpha \in [0, 1]$,
$$ \reg( s; \D, \ell ) \leq \frac{1}{( 1 - 2 \cdot \rho_{\mathrm{max}} )^{1 - \alpha}} \cdot R^{\alpha} \cdot \left( \Expectation{\X \sim M}{ w( \X ) } \right)^\alpha \cdot \left( \reg( s; \DCont, \ell ) \right)^{1 - \alpha}. $$
\end{proposition}

Since $\left( \Expectation{\X \sim M}{ w( \X ) } \right)^\alpha \leq M^{\alpha}$ trivially, the dependence on $\rho_{\mathrm{max}}$ is less stringent in the second bound above, at the expense of a possibly worse dependence on the corrupted regret.
\arxivOnly{It is intuitive that one down-weight the contribution of large weights that occur on instances with low marginal probability.}
Note that we trivially have $R = 1$ for the case of $\ell = \ellZO$, since the loss is bounded.

By Proposition \ref{prop:regret-bound}, for\footnote{For a simpler proof of Proposition \ref{prop:regret-bound} that is specific to $\ellZO$, see Appendix \ref{app:regret-01}.} $\ell = \ellZO$, if we can find a sequence $\{ s_n \}$ of scorers satisfying $\reg( s_n; \DCont, \ellZO ) \to 0$,
then we also guarantee $\reg( s; \D, \ellZO ) \to 0$, \ie we have consistency of classification on the \emph{clean} distribution.
One can guarantee $\reg( s; \DCont, \ellZO ) \to 0$ by minimising an appropriate convex surrogate to $\ellZO$ on $\DCont$, owing to standard classification calibration results \citep{Zhang:2004,Bartlett:2006}.
This surrogate does \emph{not} have to satisfy Equation \ref{eqn:loss-sum-constant}.

In the course of proving Proposition \ref{prop:risk-weighting}, we actually establish a more general relation between clean and corrupted risks (Proposition \ref{prop:general-mube} in Appendix) that holds for ILN noise, and losses not satisfying Equation \ref{eqn:loss-sum-constant}.
This general result cannot however be used to prove a meaninguful regret bound beyond the IDN case with losses satisfying Equation \ref{eqn:loss-sum-constant}; see Appendix \ref{app:clean-corrupt} for a discussion.



%
\subsection{Beyond misclassification error?}

Can the above consistency result 
be extended to generalised classification performance measures such as balanced error and $F$-score?
Disappointingly, the answer is no.
The reason is simple: 
for measures beyond the 0-1 loss, the following shows that Bayes-optimal classifier itself will not coincide on the clean and corrupted distributions, so that no analogue of Corollary \ref{corr:bayes-opt-same} can possibly hold.

\begin{proposition}
\label{prop:clean-corrupt-threshold-idn}
Pick any distribution $\D$.
Suppose that $\DCont = \IDN( \D, f )$
for admissible $f \colon \XCal \to [0, \nicefrac[]{1}{2}).$
Then, for any $t \in [0, 1]$,
$$ ( \forall x \in \XCal ) \, \eta( x ) > t \iff \etaCont( x ) > t + f( x ) \cdot (1 - 2 \cdot t). $$
\end{proposition}

For any $t \neq \nicefrac[]{1}{2}$, Proposition \ref{prop:clean-corrupt-threshold-idn} implies that classification on the corrupted distribution \emph{requires knowledge of the flipping function $f( x )$}
\ie only the 0-1 threshold is preserved under corruption.

\subsection{Relation to existing work}

The results of this section generalise those for the SLN model in \citet{Natarajan:2013}.
In particular, for the SLN model
Proposition \ref{prop:risk-weighting},
reduces to \citet[Theorem 9]{Natarajan:2013},
Corollary \ref{corr:bayes-opt-same} to \citet[Corollary 10]{Natarajan:2013}, and
Proposition \ref{prop:regret-bound} to \citet[Theorem 11]{Natarajan:2013}.

In the instance-dependent noise case, 
for linearly separable $\D$, \citet{Awasthi:2015} observed that the Bayes-optimal classifier is unchanged,
and
\citet[Theorem 1]{Ghosh:2015} established that if $\min_s R( s; \D, \ellZO ) = 0$, \ie if $\D$ is separable, then the 0-1 risk minimiser using \emph{any} function class is unaffected.
Corollary \ref{corr:bayes-opt-same} above shows that if we work with a suitably rich function class, then we have equivalence of risk minimisers even for non-separable $\D$.
\citet[Theorem 1]{Ghosh:2015} 
showed that one can make guarantees about the degradation of risk minimisation wrt losses $\ell$ satisfying Equation \ref{eqn:loss-sum-constant}.
This result only holds for the risk wrt the clean distribution $\D$, and does \emph{not} give a regret bound.
\arxivOnly{
More precisely, \citet[Theorem 1]{Ghosh:2015} (generalised to the case of general instance- and label-dependent noise in \citet{brendan_thesis}),
showed the following.

\begin{theorem}[{\citep[Theorem 1]{Ghosh:2015}}]
\label{thm:ghosh}
Pick any distribution $\D$ and loss $\ell$ satisfying Equation \ref{eqn:loss-sum-constant}.
Let $\DCont = \IDN( \D, f )$ for some admissible $f \colon \XCal \to [0, \nicefrac[]{1}{2})$.
Then, for any function class $\SCal \subseteq \Real^{\XCal}$,
$$ R( \Contaminator{s}^*; \D, \ell ) \leq \frac{R( s^*; \D, \ell )}{1 - 2 \cdot \max_{x} f( x )} $$
where
\begin{align*}
   s^* &= \underset{s \in \SCal}{\argmin} R( s; \D, \ell ) \\
   \Contaminator{s}^* &= \underset{s \in \SCal}{\argmin} R( s; \IDN( \D, f ), \ell ).
 \end{align*}
\end{theorem}

Theorem \ref{thm:ghosh} implies that for instance-dependent noise, the $\ell$-risk minimiser (for suitable $\ell$) will not differ considerably on the clean and corrupted samples.
But a limitation of the result is that one cannot guarantee \emph{consistency} wrt, \eg 0-1 loss, of using the result of $\ell$-risk minimisation on the corrupted samples.
This is because the above only holds for the risk wrt the clean distribution $\D$.
It does \emph{not} let us bound the clean regret $\reg( s; \D, \ell )$ in terms of the corrupted regret $\reg( s; \ILN( \D, \rho ), \ell )$.
}

Proposition \ref{prop:clean-corrupt-threshold-idn} generalises \citet[Theorem 9]{Natarajan:2013}, \citet[Section F.1]{Menon:2015} which were for the CCN and SLN models.
For the CCN model, \citet{Menon:2015} established that the Bayes-optimal classifier for the \emph{balanced error} is also unaffected by corruption.
One might expect this to carry over to the instance-dependent case,
but perhaps surprisingly, this is not the case.

\section{AUC consistency under the BCN$^{+}$ model}

Bipartite ranking is concerned with the ranking risk 
$$ \RRank( s; \D ) = \Expectation{\X \sim P, \X' \sim Q}{ \ellZO_1( s( \X ) - s( \X' ) ) } $$
\emph{viz.} one minus the area under the ROC curve (AUC) \citep{Agarwal:2005}.
Can an analogous regret bound to Proposition \ref{prop:regret-bound} be established for this risk?
Unfortunately, without further assumptions, this is not possible. 
The reason is as before: to establish a regret bound, the Bayes-optimal scorers must coincide.
As the AUC is optimised by any scorer that is order preserving for $\eta$ \citep{Clemencon:2008},
the corrupted AUC will be optimised by any scorer that is order preserving for $\etaCont$.
For the two to coincide, we will have to ensure that $\etaCont$ is order preserving for $\eta$.
Intuitively, this will not be true in general, since there is no necessary relationship between $\rho_{\pm 1}$ and $\eta$;
see Appendix \ref{app:order-fails}.

\subsection{Relating clean and corrupted AUC regret under the BCN$^+$ model}
\label{sec:order-preservation}

It is of interest to determine conditions under which we \emph{can} guarantee order preservation of $\eta$.
Intuitively, this will require there being \emph{some} dependence between the flip functions and $\eta$.
Fortunately, the previously introduced $\BCNPlus$ model is a feasible candidate.

\begin{proposition}
\label{prop:eta-monotone}
Pick any distribution $\D$.
Suppose $\DCont = \BCNPlus( \D, f_{-1}, f_{1}, s )$ where $( f_{-1}, f_{1}, s, \eta )$ are $\BCNPlus$-admissible.
Then,
$$ ( \forall x, x' \in \XCal ) \, \eta( x ) < \eta( x' ) \implies \etaCont( x ) < \etaCont( x' ) $$
so that $\eta = \phi \circ \etaCont$ for some non-decreasing $\phi$.
\end{proposition}

Proving Proposition \ref{prop:eta-monotone} relies on establishing a relation between $\etaCont( x ) - \etaCont( x' )$ and its counterpart on the clean distribution\footnote{For some special cases, the proofs simplify considerably; see Appendix \ref{app:simplified-monotone}.}.
We emphasise that Condition (c) in the $\BCNPlus$ model is vital to the result; see Appendix \ref{app:order-fails} for an example where removing this condition leads to a forfeit of order preservation.

Using Proposition \ref{prop:eta-monotone}, we can deduce that the $\BCNPlus$ model affords a suitable AUC regret bound.

\begin{proposition}
\label{prop:auc-regret}
Pick any distribution $\D$.
Suppose that $\DCont = \BCNPlus( \D, f_{-1}, f_{1}, s )$ where $( f_{-1}, f_{1}, s, \eta )$ are $\BCNPlus$-admissible,
and the total noise bound (Assumption \ref{ass:total-noise}) is
$$ \rho_{\mathrm{max}} = \frac{1}{2} \cdot \max_{x \in \XCal} ( \rho_1( x ) + \rho_{-1}( x ) ) < \frac{1}{2}. $$
Then, for any scorer $\scorer$,
$$ \reg_{\mathrm{rank}}( s; \D ) \leq \frac{\piCont \cdot (1 - \piCont)}{\pi \cdot (1 - \pi)} \cdot \frac{1}{1 - 2 \cdot \rho_{\mathrm{max}}} \cdot \reg_{\mathrm{rank}}( s; \DCont ) $$
where $\reg_{\mathrm{rank}}$ denotes the excess ranking risk of a scorer $s$.
\end{proposition}

Thus, under the BCN$^{+}$ model, maximising AUC on the corrupted sample is consistent for maximisation on the clean sample.
As before, one can appeal to surrogate regret bounds for the AUC \citep{Agarwal:2014} to deduce that appropriate surrogate minimisation on $\DCont$ will ensure $\reg_{\mathrm{rank}}( s; \DCont ) \to 0$.

\arxivOnly{
\begin{remark}
Proposition \ref{prop:auc-regret} is slightly surprising in the sense that the AUC can be expressed as an average of the balanced error across a range of thresholds \citep{Flach:2011}.
Proposition \ref{prop:clean-corrupt-threshold-idn} suggested that in general we do cannot have a regret bound for the clean and corrupted balanced errors.
Note that such a bound would imply one for the AUC, but not vice-versa.
\end{remark}
}

\subsection{Relation to existing work}

Proposition \ref{prop:eta-monotone} is, to our knowledge, novel.
Proposition \ref{prop:auc-regret} generalises \citet[Corollary 3]{Menon:2015}, which established a risk \emph{equivalence} between the clean and corrupted AUC for class-conditional noise.
The reason an equivalence is possible in the CCN setting is that here, $\etaCont( x ) - \etaCont( x )$ is just a scaling of $\eta( x ) - \eta( x' )$, so that both bounds in the proof above are tight (see Example \ref{ex:ccn-eta-monotone} in Appendix).

\section{Learning noisy SIMs with the Isotron}

While the preceding sections establish consistency of corrupted risk minimisation, a practical concern is how precisely one ensures vanishing regret on the corrupted distribution.
Certainly this is possible if one chooses $s$ from the set of all measurable scorers (\eg by employing a universal kernel with appropriately tuned parameters), but this may be infeasible in practical scenarios demanding the use of some restricted function class, \eg linear scorers in a low-dimensional feature space.

We turn our attention to providing a simple algorithm guaranteeing $\reg( s; \DCont, \ellZO ) \to 0$ and $\reg_{\mathrm{rank}}( s; \DCont ) \to 0$
when the class of linear scorers is suitable for $\D$, and the noise possesses some structure.
Specifically, we focus on $\D$ such that $\eta \in \SIM$,
so that $\eta = \GLM( u, w^* )$ for some (unknown) $u, w^*$.
We then consider a $\BCNPlus$ model of the noise, with $s^*( x ) = \langle w^*, x \rangle$ determining the flip probability;
for convenience, we shall call this the single index noise or ``$\SIN$'' model.

\begin{definition}
Let $f_1, f_{-1} \colon \Real \to [ 0, 1 ]$.
Given any distribution $\D$ with $\eta = \GLM( u, w^* )$ for some $(u, w^*)$, define 
$ \SIN( \D, f_{-1}, f_{1} ) \defEq \BCNPlus( \D, f_{-1}, f_{1}, s^* ) $
where $s^* \colon x \mapsto \langle w^*, x \rangle$.
\end{definition}

A special case of the above is where $\D$ is separable with some margin (see Appendix \ref{app:sim-family}),
and one observes corrupted samples with instances closer to the separator having a higher chance of being corrupted.
This is a seemingly reasonable model when labels are provided by human annotators.
A similar model was considered in \citet{Du:2015}, where it is assumed that the link $u( \cdot )$ is known.

\subsection{Corruption runs in the SIN family}

Proposition \ref{prop:eta-monotone} established that for the $\BCNPlus$ model, $\etaCont$ is order preserving for $\eta$.
When $\eta \in \SIM$, this implies that for the $\SIN$ model with suitably Lipschitz label flipping functions, $\etaCont \in \SIM$ as well.

\begin{proposition}
\label{prop:corrupt-sim}
Pick any distribution $\D$ with $\eta \in \SIM( L, W )$.
Suppose that $\DCont = \SIN( \D, f_{-1}, f_{1} )$ where
$( f_{-1}, f_{1}, \eta )$ are $\BCNPlus$-admissible, and
$( f_{-1}, f_{1} )$ are $( L_{-1}, L_{1} )$-Lipschitz respectively.
Then, $\etaCont \in \SIM( L + L_{-1} + L_{1}, W )$.
In particular,
$ \etaCont( x ) = \Contaminator{u}( \langle w^*, x \rangle ) $
where
\begin{align}
     \label{eqn:eta-cont}
     \Contaminator{u}( z ) &= ( 1 - f_1( z ) ) \cdot u( z ) + f_{-1}( z ) \cdot ( 1 - u ( z ) ).
 \end{align}
\end{proposition}

Examples of the form of $\etaCont$ for specific $u( \cdot )$ are presented in Appendix \ref{app:corrupt-sim-examples}.

\subsection{The Isotron: an efficient algorithm to learn noisy SIMs}

Proposition \ref{prop:corrupt-sim} suggests that for a large class of noisy label problems, if one can learn a generic SIM, then the corrupted class-probability function may be estimated.
Fortunately, SIMs can be provably learned with the Isotron \citep{Kalai:2009}, and its Lipschitz variant, the SLIsotron \citep{Kakade:2011}.
The elegant algorithm 
consists of alternately updating the hyperplane 
$w$, and the link function $u$.
The latter is estimated using the PAV algorithm \citep{Ayer:1955},
which finds a solution to the isotonic regression problem:
$$ ( \hat{{u}}_1, \ldots, \hat{{u}}_m ) = \underset{{u}_1 \leq {u}_2 \leq \ldots \leq {u}_m}{\argmin}{\sum_{i = 1}^m ( y_i  - {u}_i )^2}, $$
where we assume that the $s_i$'s are ordered such that $s_1 \leq s_2 \leq \ldots \leq s_m$, \ie we wish for the $u$'s to respect the ordering of the $s$'s.
The PAV algorithm provides a nonparametric estimate of $u(\cdot)$ at the specified points.
At other points, one may use linear interpolation.
The SLIsotron algorithm is identical to the Isotron, except that one calls LPAV, a variant of PAV that obeys a Lipschitz constraint.

One can provide precise theoretical guarantees on the output of the (SL)Isotron.
Combined with Proposition \ref{prop:regret-bound}, this lets one make guarantees about classification from corrupted labels.

\begin{proposition}
\label{prop:isotron-consistency}
Pick any distribution $\D$ over $\mathbb{B}^d \times \PMOne$ with $\eta \in \SIM(L, W)$.
Suppose that $\DCont = \SIN( \D, f_{-1}, f_{1} )$ where $( f_{-1}, f_{1} )$ are Lipschitz.
Then, for $\ellSQ$ being the square loss,
we can construct 
$\hat{\etaCont}_{\mathsf{S}} \colon \XCal \to [0, 1]$ from a corrupted sample $\Contaminator{\mathsf{S}} \sim \DCont^m$ using the SLIsotron, such that
$$ \reg_{\mathrm{rank}}( \hat{\etaCont}_{\Contaminator{\mathsf{S}}}; \D ) \stackrel{\Pr}{\to} 0. $$
Further, if $f_{-1} = f_{1}$, we can construct a classifier $c_{\Contaminator{\mathsf{S}}} \colon x \mapsto \sign( 2\hat{\etaCont}_{\Contaminator{\mathsf{S}}} - 1 )$ such that
$$ \reg( c_{\Contaminator{\mathsf{S}}}; \D, \ellZO ) \stackrel{\Pr}{\to} 0. $$
\end{proposition}

A salient feature of the Isotron is that one need not know the {precise} form of either $\eta$, nor the label flipping functions.
Even if one just knows that there exists some $u$ such that $\eta = \GLM(u,w^*)$, and that the labels are subject to (Lipschitz) monotonic noise, one can estimate $\etaCont$.
Even when $u( \cdot )$ is known, $\Contaminator{u}$ will involve the typically unknown flipping functions.
Thus, the Isotron solves a non-trivial estimation problem; see Appendix \ref{app:isotron-ccn} for an illustrative example in the CCN setting.

\arxivOnly{
\begin{remark}
Suppose one knows the precise form of $u$, but does not know $w^*$.
For example, one may know that $\D$ is separable with a certain margin.
Then, under the \emph{symmetric} $\BCNPlus$ model, we can in fact infer the label flipping function as
$$ f( z ) = \frac{\Contaminator{u}( z ) - u( z )}{1 - 2 \cdot u( z )}. $$
The estimation error in this term depends wholly on the error in estimating $\Contaminator{u}$.
\end{remark}
}

\section{Experiments with Isotron and boundary-consistent noise}

We now empirically verify that the Isotron can learn SIMs subject to noise from the $\BCNPlus$ model. 

\subsection{Synthetic data}
\label{sec:synthetic}

We first consider a $\D$ such that $M$ is a mixture of 2D Gaussians with identity covariance, and means $(1, 1)$ and $(-1, -1)$.
We picked $\eta \colon x \mapsto \indicator{s^*( x ) > 0}$
where $s^*( x ) = x_1 + x_2$.
For flip functions $f_{\pm 1} \colon z \mapsto 1/(1 + e^{|z|})$,
we drew a sample $\Contaminator{\SSf}$ of 5000 elements from $\DCont = \BCNPlus( \D, f_{-1}, f_{1}, s^*)$, the boundary-consistent corruption of $\D$.
We then estimated $\etaCont$ from $\Contaminator{\SSf}$ using 1000 iterations of Isotron.
Figure \ref{fig:noisetron_synthetic} shows this estimate closely matches the actual $\etaCont$ (computed explicitly via Equation \ref{eqn:corrupt-eta-general}).
Further, on a test sample of 5000 instances from $\D$,
thresholding our estimate around $\nicefrac[]{1}{2}$ gives essentially perfect ($99.46\%$) accuracy.

\begin{figure}[!h]
  \centering
  \includegraphics[scale=0.0725]{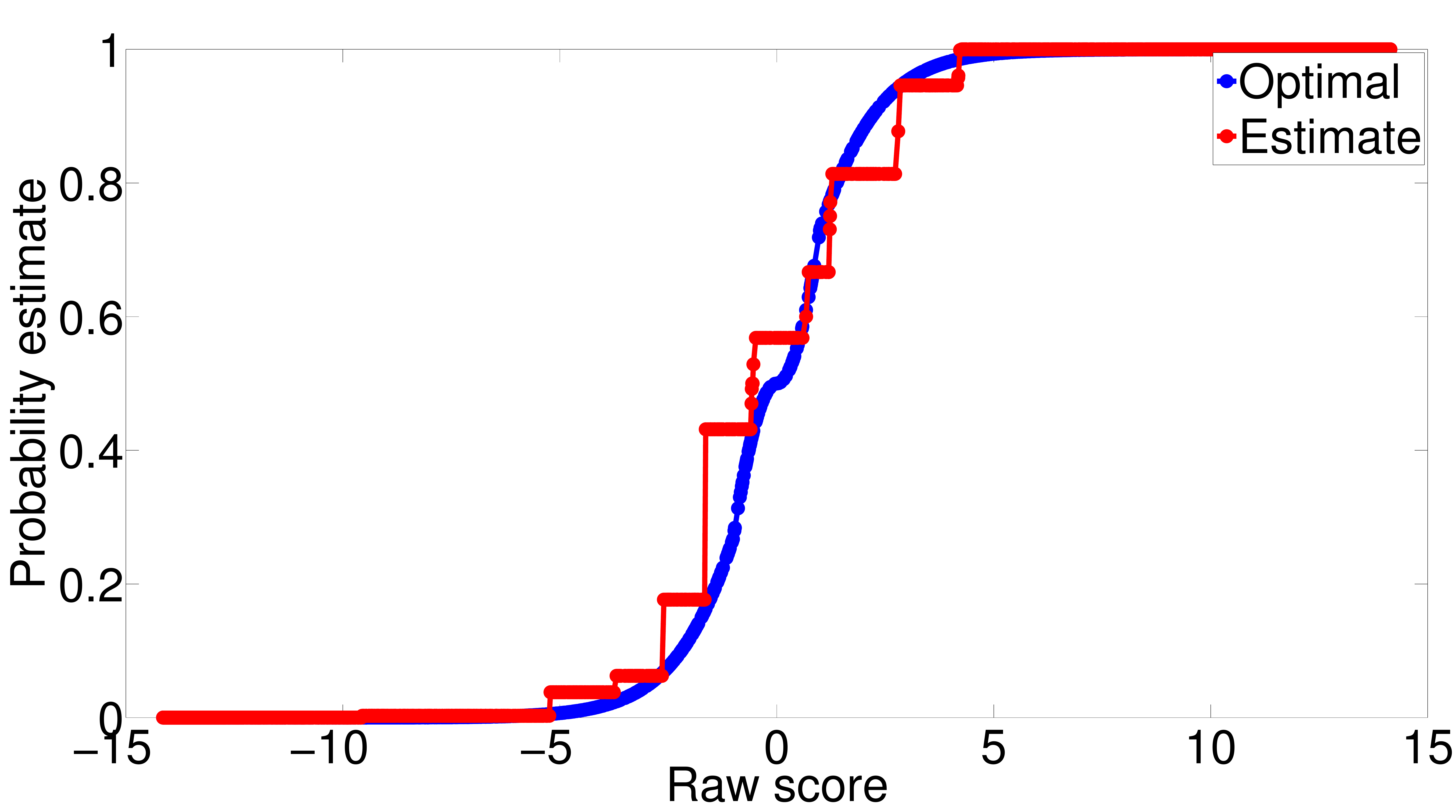}
  \caption{Isotron results for estimating $\etaCont$ on synthetic data (\S\ref{sec:synthetic}).}
  \label{fig:noisetron_synthetic}
\end{figure}

\subsection{Real-world data}
\label{sec:real}

We next ran experiments involving boundary consistent noise on the USPS and MNIST handwritten digit datasets.
We converted both datasets into binary classification tasks by seeking to distinguish digits 0 and 9 for the former, and 6 and 7 for the latter.
In each case, the binary classification task is nearly linearly separable:
to make it fully separable, we took the optimal hyperplane $w^*$ found by ordinary least squares, and discarded all instances with margin violations for a margin $\gamma = 0.1$.

On the resulting separable dataset, we created a training sample $\SSf$ comprising 80\% of instances, with the remaining 20\% of instances used for testing.
For each $(x, y) \in \SSf$, we inject 
boundary-consistent noise by flipping the label with probability $f( x ) = \left( 1 + e^{\alpha \cdot |\langle w^*, x \rangle|} \right)^{-1}$ for some parameter $\alpha \in [0, \infty)$.
The resulting corrupted sample $\Contaminator{\SSf}$ mimics a scenario where the labels are provided by a human annotator more liable to make errors for the easily confusable digits.
We then trained a regularised least squares model (using regularisation strength $\lambda = 10^{-8}$), and the Isotron (using $100$ iterations) on $\Contaminator{\SSf}$. 
We measured the models' classification accuracy on the test set with \emph{clean} labels.


Table \ref{tbl:results-mnist-usps} reports the mean and standard error of the accuracies of ridge regression and the Isotron over $T = 25$ independent label corruptions for both datasets.
We report results for $\alpha \in \{ 2^{-3}, \ldots, 2^{3} \}$, and for each $\alpha$ note the \% of labels that end up being flipped.
We find that for lower $\alpha$, both methods perform comparably.
This is unsurprising, as under low noise both should eventually find the optimal separator.
For higher $\alpha$, the Isotron offers a significant improvement over ridge regression (upto 17\% on MNIST),
in keeping with our analysis that it can effectively learn from instance-dependent noise.

\begin{table}[!h]
  \centering
  \renewcommand{\arraystretch}{1.25}

  \scalebox{0.6}{
  \subfloat[USPS 0 vs 9]{
  \begin{tabular}{llll}
    \toprule
    \toprule
    $\alpha$ & \textbf{Flip \%} & \textbf{Ridge ACC} & \textbf{Isotron ACC} \\
    \toprule
$\nicefrac[]{1}{8}$ & 0.03 $\pm$ 0.01 & 0.9940 $\pm$ 0.0003 & \cellcolor{gray!25}{0.9974 $\pm$ 0.0002} \\
$\nicefrac[]{1}{4}$ & 0.17 $\pm$ 0.01 & 0.9947 $\pm$ 0.0004 & \cellcolor{gray!25}{0.9974 $\pm$ 0.0003} \\
$\nicefrac[]{1}{2}$ & 2.15 $\pm$ 0.09 & \cellcolor{gray!25}{0.9944 $\pm$ 0.0004} & 0.9937 $\pm$ 0.0006 \\
1 & 11.84 $\pm$ 0.17 & \cellcolor{gray!25}{0.9853 $\pm$ 0.0012} & 0.9700 $\pm$ 0.0021 \\
2 & 26.57 $\pm$ 0.22 & 0.8988 $\pm$ 0.0053 & \cellcolor{gray!25}{0.9239 $\pm$ 0.0050} \\
4 & 37.65 $\pm$ 0.24 & 0.7410 $\pm$ 0.0072 & \cellcolor{gray!25}{0.7863 $\pm$ 0.0138} \\
8 & 43.76 $\pm$ 0.25 & 0.6185 $\pm$ 0.0078 & \cellcolor{gray!25}{0.6467 $\pm$ 0.0405} \\
    \bottomrule
  \end{tabular}
  }
  \subfloat[MNIST 6 vs 7]{
  \begin{tabular}{llll}
    \toprule
    \toprule
    $\alpha$ & \textbf{Flip \%} & \textbf{Ridge ACC} & \textbf{Isotron ACC} \\
    \toprule
$\nicefrac[]{1}{8}$ & 0.04 $\pm$ 0.00 & 0.9958 $\pm$ 0.0001 & \cellcolor{gray!25}{0.9984 $\pm$ 0.0001} \\
$\nicefrac[]{1}{4}$  & 0.44 $\pm$ 0.01 & 0.9958 $\pm$ 0.0001 & \cellcolor{gray!25}{0.9979 $\pm$ 0.0001} \\
$\nicefrac[]{1}{2}$   & 4.25 $\pm$ 0.04 & 0.9953 $\pm$ 0.0002 & \cellcolor{gray!25}{0.9966 $\pm$ 0.0003} \\
1     & 15.97 $\pm$ 0.05 & \cellcolor{gray!25}{0.9871 $\pm$ 0.0005} & 0.9864 $\pm$ 0.0007 \\
2     & 29.97 $\pm$ 0.09 & 0.9446 $\pm$ 0.0012 & \cellcolor{gray!25}{0.9565 $\pm$ 0.0013} \\
4     & 39.49 $\pm$ 0.08 & 0.8262 $\pm$ 0.0022 & \cellcolor{gray!25}{0.8768 $\pm$ 0.0041} \\
8     & 44.63 $\pm$ 0.08 & 0.6872 $\pm$ 0.0024 & \cellcolor{gray!25}{0.8088 $\pm$ 0.0291} \\
    \bottomrule
  \end{tabular}
  }
  }
  
  \caption{Mean and standard error for 0-1 accuracies of ridge regression (``Ridge'') and Isotron over $T = 25$ independent corruption trials. See text in \S\ref{sec:real} for details of parameter $\alpha$.}
  \label{tbl:results-mnist-usps}
\end{table}

\section{Conclusion}

We have analysed the problem of learning with instance- and label-dependent noise, concluding that
for instance-dependent noise, minimising the classification risk on the noisy distribution is consistent for classification on the clean distribution;
for a broad class of instance- and label-dependent noise, a similar consistency result holds for the area under the ROC curve;
and one can learn generalised linear models subject to the same noise model using the Isotron.

\newpage

\bibliographystyle{plainnat}
\bibliography{references}

\newpage

\appendix
\let\stdsection\section
\renewcommand\section{\newpage\stdsection}

{\LARGE
\begin{center}
\textbf{Supplementary Material for ``Learning from Binary Labels with Instance-Dependent Corruption''}
\end{center}
}

Appendix \ref{sec:helper} contains some helper results useful for the proofs of results in the main body.
These proofs are provided in Appendix \ref{sec:proofs}.
The subsequent sections contain further discussion and examples of material from the main body.
(There are no new theoretical results in these later sections, but rather, simply expositions of special cases of results in the main body.)

\newpage

\section{Additional helper results}
\label{sec:helper}

%
\subsection{Order preservation}

We will make use of the following simple fact about order preservation, stated without proof.

\begin{lemma}
Suppose $f, g \colon \Real \to \Real$ are such that
$$ ( \forall x, y \in \Real ) \, f( x ) < f( y ) \implies g( x ) < g( y ). $$
Then, $f = u \circ g$ for some non-decreasing $u$.
\end{lemma}

Taking the contrapositive gives us an alternate useful statement.

\begin{corollary}
\label{corr:monotone-contrapositive}
Suppose $f, g \colon \Real \to \Real$ are such that
$$ ( \forall x, y \in \Real ) \, g( x ) \leq g( y ) \implies f( x ) \leq f( y ). $$
Then, $f = u \circ g$ for some non-decreasing $u$.
\end{corollary}

Finally, we can make a more precise statement about behaviour when $g( x ) = g( y )$ under the above conditions.

\begin{lemma}
Suppose $f, g \colon \Real \to \Real$ are such that
$$ ( \forall x, y \in \Real ) \, f( x ) < f( y ) \implies g( x ) < g( y ). $$
Then,
$$ ( \forall x, y \in \Real ) \, g( x ) = g( y ) \implies f( x ) = f( y ). $$
$$ ( \forall x, y \in \Real ) \, g( x ) < g( y ) \implies f( x ) \leq f( y ). $$
\end{lemma}

\begin{proof}
By the contrapositive,
$$ ( \forall x, y \in \Real ) \, g( x ) \leq g( y ) \implies f( x ) \leq f( y ). $$
If $g( x ) < g( y )$ then trivially $g( x ) \leq g( y )$ and the result follows.
Suppose that $g( x ) = g( y )$.
Then $g( x ) \leq g( y )$ and $g( y ) \leq g( x )$.
Thus $f( x ) \leq f( y )$ and $f( y ) \leq f( x )$, \ie $f( x ) = f( y )$.
\end{proof}

Note that if we only know that $g( x ) < g( y ) \implies f( x ) \leq f( y )$, we cannot conclude that $f = u \circ g$, nor that $g = u \circ f$;
we must be able to conclude something about the behaviour of $f$ when $g( x ) = g( y )$.

\subsection{Additional properties of $\DCont$}

From Proposition \ref{lemm:corrupt-eta-general}, we can derive expressions for the corrupted base rate and class-conditional distributions.

\begin{corollary}
Pick any distribution $\D$.
Then, for any $\rho_1, \rho_{-1} \colon \XCal \to [ 0, 1 ]$,
$\ILN( \D, \rho_{-1}, \rho_{1} )$ has
\begin{align*}
    \piCont &= \pi - \Expectation{\X \sim M}{ ( \rho_1( \X ) + \rho_{-1}( \X ) ) \cdot \eta( \X ) } + \Expectation{\X \sim M}{ \rho_{-1}( \X ) ) } \\
    \PCont( x ) &= \piCont^{-1} \cdot \left( (1 - \rho_1( x ) ) \cdot \pi \cdot P( x ) + \rho_{-1}( x ) \cdot (1 - \pi) \cdot Q( x ) \right) \\
    \QCont( x ) &= (1 - \piCont)^{-1} \cdot \left( \rho_1( x ) \cdot \pi \cdot P( x ) + (1 - \rho_{-1}( x )) \cdot (1 - \pi) \cdot Q( x ) \right).
\end{align*}
\end{corollary}

\begin{proof}
Proposition \ref{lemm:corrupt-eta-general} implies that the corrupted base rate is
\begin{align*}
     \piCont &= \Expectation{\X \sim M}{ (1 - \rho_1( \X ) ) \cdot \eta( \X ) + \rho_{-1}( \X ) \cdot ( 1 - \eta( \X ) ) } \\
     &= \pi - \Expectation{\X \sim M}{ ( \rho_1( \X ) + \rho_{-1}( \X ) ) \cdot \eta( \X ) } + \Expectation{\X \sim M}{ \rho_{-1}( \X ) ) },
 \end{align*}
which is a complex translation of the clean base rate.
 Further, the corrupted class-conditional distributions are
 \begin{align*}
     \PCont( x ) &= \frac{\etaCont( x ) \cdot M( x )}{\piCont} \\
     &= \frac{((1 - \rho_1( x ) - \rho_{-1}(x) ) \cdot \eta( x ) + \rho_{-1}( x ) ) \cdot M( x )}{\piCont} \\
     &= ( 1 - \rho_1( x ) - \rho_{-1}(x) ) \cdot P( x ) \cdot \frac{\pi}{\piCont} + \frac{\rho_{-1}( x ) \cdot M( x )}{\piCont} \\
     &= \piCont^{-1} \cdot \left( (1 - \rho_1( x ) ) \cdot \pi \cdot P( x ) + \rho_{-1}( x ) \cdot (1 - \pi) \cdot Q( x ) \right),
 \end{align*}
 and similarly
 $$ \QCont( x ) = (1 - \piCont)^{-1} \cdot \left( \rho_1( x ) \cdot \pi \cdot P( x ) + (1 - \rho_{-1}( x )) \cdot (1 - \pi) \cdot Q( x ) \right). $$
 For the class-conditionals, we can equally write
\begin{equation}    
\label{eqn:class-conditional-relation}
\begin{aligned}
     P( x ) &= (1 - \rho_1( x ) - \rho_{-1}( x ))^{-1} \cdot \pi^{-1} \cdot \left( (1 - \rho_{-1}( x )) \cdot \piCont \cdot \PCont( x ) - \rho_{-1}( x ) \cdot (1 - \piCont) \cdot \QCont( x ) \right) \\
     Q( x ) &= (1 - \rho_1( x ) - \rho_{-1}( x ))^{-1} \cdot (1 - \pi)^{-1} \cdot \left( -\rho_1( x ) \cdot \piCont \cdot \PCont( x ) + (1 - \rho_{1}( x )) \cdot (1 - \piCont) \cdot \QCont( x ) \right).
 \end{aligned}
 \end{equation}
\end{proof}

\subsection{Relating clean and corrupt risks}
\label{app:clean-corrupt}

We have the following general relationship between the risk on the clean and corrupted distributions.

\begin{proposition}
\label{prop:general-mube}
Pick any distribution $\D$, and
any loss $\ell$.
Suppose that $\DCont = \ILN( \D, \rho_1, \rho_{-1} )$
for admissible $\rho_{\pm 1} \colon \XCal \to [0, 1]$.
Then, for any scorer $\scorer$,
\begin{align*}
    R( s; \D, \ell ) = \Expectation{(\X, \Contaminator{\Y}) \sim \DCont}{ \tilde{\ell}( \Contaminator{\Y}, s, \X ) }
\end{align*}
where $\tilde{\ell} \colon \PMOne \times \Real^{\XCal} \times \XCal \to \Real$ is a ``generalised loss'' given by
\begin{align*}
    \tilde{\ell}_1( s, x ) &= w( x ) \cdot \left( (1 - \rho_{-1}( x )) \cdot \ell_1( s( x ) ) - \rho_{1}( x ) \cdot \ell_{-1}( s( x ) ) \right) \\
    \tilde{\ell}_{-1}( s, x ) &= w( x ) \cdot \left( -\rho_{-1}( x ) \cdot \ell_1( s( x ) ) + (1 - \rho_{1}( x ) ) \cdot \ell_{-1}( s( x ) ) \right)
\end{align*}
where $w( x ) = (1 - \rho_1( x ) - \rho_{-1}( x ))^{-1}$.
\end{proposition}

\begin{proof}[Proof of Proposition \ref{prop:general-mube}]
By Proposition \ref{lemm:corrupt-eta-general}, for $\ILN( \D, \rho_1, \rho_{-1} )$,
$$ ( \forall x \in \XCal ) \, \eta( x ) = \frac{\etaCont( x ) - \rho_{-1}( x )}{w( x )} $$
and
$$ ( \forall x \in \XCal ) \, 1 - \eta( x ) = \frac{1 - \etaCont( x ) - \rho_{1}( x )}{w( x )}, $$
where $w( x ) = (1 - \rho_1( x ) - \rho_{-1}( x ))^{-1}$.
Thus, the $\ell$-risk of an arbitrary scorer is
\begin{align}
    \label{eqn:risk-intermediate}
    \nonumber R( s; \D, \ell ) =&\ \Expectation{\X \sim M}{ L( \eta( \X ), s( \X ) ) } \\
    \nonumber =&\ \Expectation{\X \sim M}{ \eta( \X ) \cdot \ell_{1}( s( \X ) ) + (1 - \eta( \X ) \cdot \ell_{-1}( s( \X ) )) } \\
    =&\ \Expectation{\X \sim M}{ w(\X)^{-1} \cdot \left( (\etaCont( \X ) - \rho_{-1}( \X )) \cdot \ell_{1}( s( \X ) ) + (1 - \etaCont( \X ) - \rho_{1}( \X ) ) \cdot \ell_{-1}( s( \X ) ) ) \right) }.
\end{align}

Observe that this may be re-expressed as
\begin{align*}
    R( s; \D, \ell ) =&\ \Expectation{\X \sim M}{ w(\X)^{-1} \cdot \left( (\etaCont( \X ) - ( \etaCont( \X ) + (1 - \etaCont( \X )) ) \cdot \rho_{-1}( \X ) ) \cdot \ell_{1}( s( \X ) \right) } + \\
    & \Expectation{\X \sim M}{ w(\X)^{-1} \cdot \left( (1 - \etaCont( \X ) - ( \etaCont( \X ) + (1 - \etaCont( \X )) ) \cdot \rho_{1}( \X ) ) \cdot \ell_{-1}( s( \X ) ) ) \right) } \\
    =&\ \Expectation{\X \sim M}{ \etaCont( \X ) \cdot w(\X)^{-1} \cdot ( (1 - \rho_{-1}( \X )) \cdot \ell_1( s( \X ) ) - \rho_{1}( \X ) \cdot \ell_{-1}( s( \X ) ) ) } + \\ 
    &\Expectation{\X \sim M}{ (1 - \etaCont( \X )) \cdot w(\X)^{-1} \cdot ( -\rho_{-1}( \X ) \cdot \ell_1( s( \X ) ) + (1 - \rho_{1}( \X ) ) \cdot \ell_{-1}( s( \X ) ) ) } \\    
    =&\ \Expectation{(\X, \Contaminator{\Y}) \sim \DCont}{ \tilde{\ell}( \Contaminator{\Y}, s, \X ) }.
\end{align*}
\end{proof}

Proposition \ref{prop:general-mube} is a generalisation of \citet[Lemma 1]{Natarajan:2013}.
For the CCN setting, the ``generalised loss'' object of this Proposition simplifies to the ``noise-corrected loss'' studied in \citet{Natarajan:2013}, with Proposition \ref{prop:general-mube} simply being the ``method of unbiased estimators'' described in that paper (see Appendix \ref{app:generalised-loss}).

Note that Proposition \ref{prop:general-mube} cannot be used to establish a regret bound in general.
This is because the ``generalised loss'' above only simplifies to a weighted version of $\ell$ under very specific cases
(with an example being the IDN model and the partial losses summing to a constant).

%
\subsection{Relating clean and corrupt thresholds}

For a general ILN model, we have the following.

\begin{proposition}
\label{prop:clean-corrupt-threshold}
Pick any distribution $\D$.
Suppose that $\DCont = \ILN( \D, \rho_{-1}, \rho_{1} )$
for admissible $\rho_{\pm 1} \colon \XCal \to [0, 1].$
Then, for any $t \in [0, 1]$,
$$ ( \forall x \in \XCal ) \, \eta( x ) > t \iff \etaCont( x ) > ( 1 - \rho_{1}( x ) - \rho_{-1}( x ) ) \cdot t + \rho_{-1}( x ). $$
\end{proposition}

\begin{proof}[Proof of Proposition \ref{prop:clean-corrupt-threshold}]
By Proposition \ref{lemm:corrupt-eta-general},
$$ \eta( x ) = \frac{\etaCont( x ) - \rho_{-1}( x )}{1 - \rho_{1}( x ) - \rho_{-1}( x )}. $$
Now if $\rho_1( x ) + \rho_{-1}( x ) < 1$ for every $x$, $1 - \rho_1( x ) - \rho_{-1}( x ) > 0$.
We thus have
\begin{align*}
     \eta( x ) > t &\iff \frac{\etaCont( x ) - \rho_{-1}( x )}{1 - \rho_{1}( x ) - \rho_{-1}( x )} > t \\
     &\iff {\etaCont( x ) - \rho_{-1}( x )} > {(1 - \rho_{1}( x ) - \rho_{-1}( x ) \cdot t} \text{ since } 1 - \rho_1( x ) - \rho_{-1}( x ) > 0 \\
     &\iff {\etaCont( x )} > ( 1 - \rho_{1}( x ) - \rho_{-1}( x ) ) \cdot t + \rho_{-1}( x ).
\end{align*}
\end{proof}

\subsection{Difference in $\etaCont$ values}

For the general ILN model, we have the following relation between the difference in $\etaCont$ values and the corresponding $\eta$ values.

\begin{lemma}
\label{lemm:eta-diff}
Pick any distribution $\D$.
Suppose $\DCont = \ILN( \D, \rho_{-1}, \rho_{1} )$.
Then,
\begin{align*}
    ( \forall x, x' \in \XCal ) \, \etaCont( x ) - \etaCont( x' ) &= ( 1 - \rho_{-1}( x' ) - \rho_{1}( x' ) ) \cdot ( \eta( x ) - \eta( x' ) ) + \Delta_1( x, x' ) \\
    &= ( 1 - \rho_{-1}( x ) - \rho_{1}( x ) ) \cdot ( \eta( x ) - \eta( x' ) ) + \Delta_2( x, x' ),
\end{align*}
where
\begin{align*}
     \Delta_1( x, x' ) &= (\rho_{-1}( x ) - \rho_{-1}( x' )) \cdot (1 - \eta( x )) + (\rho_{1}( x' ) - \rho_{1}( x )) \cdot \eta( x ) \\
     \Delta_2( x, x' ) &= (\rho_{-1}( x ) - \rho_{-1}( x' )) \cdot (1 - \eta( x' )) + (\rho_{1}( x' ) - \rho_{1}( x )) \cdot \eta( x' ).
\end{align*}
\end{lemma}

\begin{example}
\label{ex:ccn-eta-monotone}
For the case of CCN learning $\CCN( \D, \rho_{-1}, \rho_{1} )$, $\Delta_1 \equiv \Delta_2 \equiv 0$ and so we have the simpler expression
$$ \etaCont( x ) - \etaCont( x' ) = (1 - \alpha - \beta) \cdot ( \eta( x ) - \eta( x' ) ), $$
from which order preservation is immediate.
\end{example}

\begin{example}
For the case of IDN learning $\IDN( \D, f )$,
\begin{align*}
 \Delta_1( x, x' ) &= ( f( x ) - f( x' ) ) \cdot ( 1 - 2 \cdot \eta( x ) ) \\
 \Delta_2( x, x' ) &= ( f( x ) - f( x' ) ) \cdot ( 1 - 2 \cdot \eta( x' ) ).
\end{align*}
Thus,
\begin{align*}
   \etaCont( x ) - \etaCont( x' ) &= (1 - 2 \cdot f( x' )) \cdot (\eta( x ) - \eta( x' )) + ( f( x ) - f( x' ) ) \cdot ( 1 - 2 \cdot \eta( x ) ) \\
   &= (1 - 2 \cdot f( x )) \cdot (\eta( x ) - \eta( x' )) + ( f( x ) - f( x' ) ) \cdot ( 1 - 2 \cdot \eta( x' ) ).
 \end{align*}
 Order preservation here will depend on the structure of $f$.
\end{example}

\begin{proof}[Proof of Lemma \ref{lemm:eta-diff}]
By Proposition \ref{lemm:corrupt-eta-general},
\begin{align*}
    \etaCont( x ) &= ( 1 - \rho_1( x ) - \rho_{-1}( x ) ) \cdot \eta( x ) + \rho_{-1}( x ).
\end{align*}
We have
\begin{align}
    \label{eqn:phi-delta1}
    \nonumber
    \etaCont( x ) - \etaCont( x' ) &= 
    ( 1 - \rho_{-1}( x ) - \rho_{1}( x ) ) \cdot \eta( x ) - ( 1 - \rho_{-1}( x' ) - \rho_{1}( x' ) ) \cdot \eta( x' ) + \rho_{-1}( x ) - \rho_{-1}( x' ) \\
    &= ( 1 - \rho_{-1}( x' ) - \rho_{1}( x' ) ) \cdot ( \eta( x ) - \eta( x' ) ) + \Delta_1( x, x' ),
\end{align}
where
\begin{align*}
     \Delta_1( x, x' ) &= (\rho_{-1}( x' ) + \rho_{1}( x' ) - \rho_{-1}( x ) - \rho_{1}( x )) \cdot \eta( x ) + (\rho_{-1}( x ) - \rho_{-1}( x' )) \\
     &= (\rho_{-1}( x ) - \rho_{-1}( x' )) \cdot (1 - \eta( x )) + (\rho_{1}( x' ) - \rho_{1}( x )) \cdot \eta( x );
 \end{align*}
alternately, we have
\begin{align}
    \label{eqn:phi-delta2}
    \etaCont( x ) - \etaCont( x' ) &= ( 1 - \rho_{-1}( x ) - \rho_{1}( x ) ) \cdot ( \eta( x ) - \eta( x' ) ) + \Delta_2( x, x' )
\end{align}
where
\begin{align*}
    \Delta_2( x, x' ) &= (\rho_{-1}( x' ) - \rho_{-1}( x ) + \rho_{1}( x' ) - \rho_{1}( x )) \cdot \eta( x' ) + (\rho_{-1}( x ) - \rho_{-1}( x' )) \\
    &= (\rho_{-1}( x ) - \rho_{-1}( x' )) \cdot (1 - \eta( x' )) + (\rho_{1}( x' ) - \rho_{1}( x )) \cdot \eta( x' ).
\end{align*}
\end{proof}

For the BCN$^+$ model, Lemma \ref{lemm:eta-diff} can be converted to show that $\etaCont$ is a monotone transform of $s$, the underlying score used in the noise model;
furthermore, we have a simple bound on the differences in $\etaCont$ values in terms of the corresponding difference in $\eta$ values.

\begin{lemma}
\label{lemm:eta-diff-bound}
Pick any distribution $\D$.
Suppose $\DCont = \BCNPlus( \D, f_{-1}, f_{1}, s )$ where $( f_{-1}, f_{1}, s, \eta )$ are $\BCNPlus$-admissible.
Then,
$$ ( \forall x, x' \in \XCal ) \, s( x ) \leq s( x' ) \implies \etaCont( x ) - \etaCont( x' ) \leq \max( 1 - \rho_{-1}( x ) - \rho_{1}( x ), 1 - \rho_{-1}( x' ) - \rho_{1}( x' ) ) \cdot ( \eta( x ) - \eta( x' ) ) $$
where $\rho_{\pm 1}( x ) = f_{\pm 1} \circ s$.
\end{lemma}

\begin{proof}[Proof of Lemma \ref{lemm:eta-diff-bound}]
For the BCN model, Lemma \ref{lemm:eta-diff} is
\begin{align*}
    ( \forall x, x' \in \XCal ) \, \etaCont( x ) - \etaCont( x' ) &= ( 1 - f_{-1}( z' ) - f_{1}( z' ) ) \cdot ( u( z ) - u( z' ) ) + \Delta_1( z, z' ) \\
    &= ( 1 - f_{-1}( z ) - f_{1}( z ) ) \cdot ( u( z ) - u( z' ) ) + \Delta_2( z, z' ),
\end{align*}
where $z = s( x ), z' = s( x' )$, and
\begin{align*}
     \Delta_1( z, z' ) &= (f_{-1}( z ) - f_{-1}( z' )) \cdot (1 - u( z )) + (f_{1}( z' ) - f_{1}( z )) \cdot u( z ) \\
     \Delta_2( z, z' ) &= (f_{-1}( z ) - f_{-1}( z' )) \cdot (1 - u( z' )) + (f_{1}( z' ) - f_{1}( z )) \cdot u( z' ).
\end{align*}

Suppose that $s( x ) = s( x' )$.
Then clearly $\Delta_1 \equiv \Delta_2 \equiv 0$ and $u( z ) = u( z' )$, so $\etaCont( x ) = \etaCont( x' )$.

Suppose that $s( x ) < s( x' )$ so that\footnote{By contrapositive of Condition (a) of $\BCNPlus$-admissibility, if $s( x ) \leq s( x' )$ then $\eta( x ) \leq \eta( x' )$.} $\eta( x ) \leq \eta( x' )$; or equivalently, $z < z'$ so that $u( z ) \leq u( z' )$.
Our goal is to show that $\min( \Delta_1( z, z' ), \Delta_2( z, z' ) ) \leq 0$;
this will imply the desired bound, since we can just use the tighter of the implied bounds on Equation \ref{eqn:phi-delta1} and \ref{eqn:phi-delta2}.
By Condition (c) of $\BCNPlus$-admissibility, for any $z < z'$,
$$ f_{1}( z ) - f_{-1}( z ) \geq f_{1}( z' ) - f_{-1}( z' ) $$
or equivalently 
$$ f_{1}( z' ) - f_{1}( z ) \leq f_{-1}( z' ) - f_{-1}( z ). $$
Thus, since $u( z ) \geq 0$, we have
\begin{align}
    \label{eqn:phi-bound1}
    \Delta_1( z, z' ) &\leq 
    (f_{-1}( z ) - f_{-1}( z' )) \cdot (1 - 2 \cdot u( z )),
\end{align}
and similarly,
\begin{align}
    \label{eqn:phi-bound2}
     \Delta_2( z, z' ) &\leq (f_{-1}( z ) - f_{-1}( z' )) \cdot (1 - 2 \cdot u( z' )).
\end{align}

We now argue why the minimum of these terms must be $\leq 0$.
Consider the following three cases:
\begin{enumerate}[(a)]
    \item Suppose $f_{-1}( z ) = f_{-1}( z' )$. Then trivially both terms are $\leq 0$.

    \item Suppose $f_{-1}( z ) < f_{-1}( z' )$. Then either $u( z ) \leq \frac{1}{2}$ or $u( z' ) \leq \frac{1}{2}$; if both $u$ values are larger than $\frac{1}{2}$, then by BCN-admissibility Condition (b) it must be true that $f_{-1}( z ) \geq f_{-1}( z' )$, a contradiction. Thus either $1 - 2 \cdot u( z ) \geq 0$ or $1 - 2 \cdot u( z' ) \geq 0$, and so one of the terms must be $\leq 0$.

    \item Suppose $f_{-1}( z ) > f_{-1}( z' )$. Then either $u( z ) \geq \frac{1}{2}$ or $u( z' ) \geq \frac{1}{2}$; if both $u$ values are smaller than $\frac{1}{2}$, then by BCN-admissibility Condition (b) it must be true that $f_{-1}( z ) \leq f_{-1}( z' )$, a contradiction. Thus either $1 - 2 \cdot u( z ) \leq 0$ or $1 - 2 \cdot u( z' ) \leq 0$, and so one of the terms must be $\leq 0$.  
\end{enumerate}
Thus, we conclude $\min( \Delta_1( z, z' ), \Delta_2( z, z' ) ) \leq 0$, and so either
$$ {\etaCont( x ) - \etaCont( x' )} \leq ( 1 - \rho_{-1}( x ) - \rho_{1}( x ) ) \cdot ( \eta( x ) - \eta( x' ) ) $$
or
$$ {\etaCont( x ) - \etaCont( x' )} \leq ( 1 - \rho_{-1}( x ) - \rho_{1}( x' ) ) \cdot ( \eta( x ) - \eta( x' ) ) $$
must be true; since $\eta( x ) - \eta( x' ) \leq 0$ and $\max( 1 - \rho_{-1}( x ) - \rho_{1}( x ), 1 - \rho_{-1}( x ) - \rho_{1}( x' ) ) > 0$, this implies
$$ {\etaCont( x ) - \etaCont( x' )} \leq \max( 1 - \rho_{-1}( x ) - \rho_{1}( x ), 1 - \rho_{-1}( x ) - \rho_{1}( x' ) ) \cdot ( \eta( x ) - \eta( x' ) ). $$
Since $\eta( x ) - \eta( x' ) \leq 0$ and $\max( 1 - \rho_{-1}( x ) - \rho_{1}( x ), 1 - \rho_{-1}( x ) - \rho_{1}( x' ) ) > 0$, we may bound the entire expression by $0$, thus concluding that $\etaCont( x ) \leq \etaCont( x' )$.
\end{proof}

An immediate consequence of Lemma \ref{lemm:eta-diff-bound} is that $\etaCont$ is order-preserving for the underlying scores.

\begin{corollary}
\label{corr:eta-cont-score-preserving}
Suppose $\DCont = \BCNPlus( \D, f_{-1}, f_{1}, s )$ where $( f_{-1}, f_{1}, s, \eta )$ are $\BCNPlus$-admissible.
Then,
$$ ( \forall x, x' \in \XCal ) \, s( x ) \leq s( x' ) \implies \etaCont( x ) \leq \etaCont( x' ) $$
and so $\etaCont = \Contaminator{u} \circ s$ for some non-decreasing $\Contaminator{u}$.
\end{corollary}

\begin{proof}
By Lemma \ref{lemm:eta-diff-bound}, if $s( x ) = s( x' )$ then $\etaCont( x ) = \etaCont( x' )$.
If $s( x ) < s( x' )$ then $\eta( x ) \leq \eta( x' )$ by BCN-admissiblity Condition (a).
Further, $1 - \rho_1( x ) - \rho_{-1}( x ) > 0$ by Assumption \ref{ass:total-noise}.
Thus, $\etaCont( x ) - \etaCont( x' ) \leq 0$.

The fact that $\etaCont = \Contaminator{u} \circ s$ follows from Corollary \ref{corr:monotone-contrapositive}.
\end{proof}

\begin{remark}
By definition of BCN admissibility, $\eta = u \circ s$ for some monotone $u$;
and by Lemma \ref{lemm:eta-diff-bound}, $\etaCont = \bar{u} \circ s$, for some monotone $\bar{u}$.
If we could establish that $\bar{u}$ were \emph{strictly} monotone, then we would immediately conclude $\eta = u \circ \bar{u}^{-1} \circ \etaCont$, which would establish Proposition \ref{prop:eta-monotone}.
But this is not true in general; fortunately, $\bar{u}$ is only constant when $u$ is (owing to the explicit bound in Lemma \ref{lemm:eta-diff-bound}), and so we are still able to write $\eta = \phi \circ \etaCont$ for some monotone $\phi$.
\end{remark}

\subsection{Class-probability estimation guarantees with the Isotron}

The basic SLIsotron guarantee is as follows.

\begin{proposition}[{\citep[Theorem 2]{Kakade:2011}}]
\label{prop:isotron}
Pick any $\D$ over $\mathbb{B}^d \times \PMOne$ with\footnote{If $\eta \in \SIM( L, W )$, then trivially $\eta \in \SIM( 1, L \cdot W )$, because $\eta( x ) = u( \langle w^*, x \rangle ) = u( (1/L) \cdot \langle (L \cdot w^*), x \rangle ) = \tilde{u}( \langle \tilde{w}^*, x \rangle )$, where $\tilde{u}$ is a 1-Lipschitz function, and $|| \tilde{w}^* || = L \cdot W$.} $\eta \in \SIM( 1, W )$ for some $W \in \Real_+$.
Let $\{ \hat{\eta}_{\mathsf{S}, t} \}_{t = 1}^\infty$ denote the estimates of $\eta$ produced at each iteration of SLISotron, when applied to a training sample $\mathsf{S}$.
Then,
$$ \Pr_{\mathsf{S} \sim D^m}\left( \min_{t} \reg( \hat{\eta}_{\mathsf{S}, t}; \D, \ellSQ ) \leq \left( \frac{dW^2}{m} \right)^{1/3} \cdot \left( \log \frac{Wm}{\delta} \right)^{1/3} \right) \geq 1 - \delta $$
where
$$ \reg( \hat{\eta}; \D, \ellSQ ) = \Expectation{\X \sim M}{ ( \hat{\eta}( \X ) - \eta( \X ) )^2 }. $$
\end{proposition}

\section{Proofs of results in main body}
\label{sec:proofs}

%
\begin{proof}[Proof of Lemma \ref{lemm:corrupt-eta-general}]
By definition of how corrupted labels $\YCont$ are generated,
\begin{align*}
\etaCont( x ) &= \Pr( \Contaminator{\Y} = 1 \mid \X = x ) \\
&= \sum_{y \in \PMOne} \Pr( \Contaminator{\Y} = 1 \mid \Y = y, \X = x ) \cdot \Pr( \Y = y \mid \X = x ) \\
&= (1 - \rho_1( x ) ) \cdot \eta( x ) + \rho_{-1}( x ) \cdot ( 1 - \eta(x) ).
\end{align*}
The second identity follows by rearranging.
\end{proof}

\begin{proof}[Proof of Proposition \ref{prop:risk-weighting}]
By Equation \ref{eqn:risk-intermediate} from the proof of Proposition \ref{prop:general-mube}, for $\ILN( \D, \rho_1, \rho_{-1} )$,
\begin{align*}
    R( s; \D, \ell ) =&\ \Expectation{\X \sim M}{ w(\X)^{-1} \cdot \left( (\etaCont( \X ) - \rho_{-1}( \X )) \cdot \ell_{1}( s( \X ) ) + (1 - \etaCont( \X ) - \rho_{1}( \X ) ) \cdot \ell_{-1}( s( \X ) ) ) \right) } \\
    =&\ \Expectation{\X \sim M}{ w(\X)^{-1} \cdot ( \etaCont( \X ) \cdot \ell_1( s(\X) ) + ( 1 - \etaCont( \X ) ) \cdot \ell_{-1}( s( \X ) ) ) ) } - \\
    &\Expectation{\X \sim M}{ w(\X)^{-1} \cdot \left( \rho_{-1}( \X ) \cdot \ell_{1}( s( \X ) ) + \rho_{1}( \X ) \cdot \ell_{-1}( s( \X ) ) \right) } \\
    =&\ \Expectation{\X \sim M}{ w(\X)^{-1} \cdot \left( L( \etaCont( \X ), s( \X ) ) - \rho_{-1}( \X ) \cdot \ell_{1}( s( \X ) ) + \rho_{1}( \X ) \cdot \ell_{-1}( s( \X ) ) \right) }.
\end{align*}
If $\rho_1 \equiv \rho_{-1} \equiv f$, $w( x ) = 1 - 2 \cdot f( x )$ and
\begin{align*}
    R( s; \D, \ell ) &= \Expectation{\X \sim M}{ \frac{1}{1 - 2 \cdot f( \X )} \cdot L( \etaCont( \X ), s( \X ) ) } - \Expectation{\X \sim M}{ \frac{f( \X )}{1 - 2 \cdot f( \X )} \cdot ( \ell_{1}( s( \X ) ) + \ell_{-1}( s( \X ) ) ) }.
\end{align*}
Thus, if the sum of the partial losses is a constant $C$,
$$ R( s; \D, \ell ) = R^{\weight( w )}( s; \DCont, \ell ) - C \cdot \Expectation{\X \sim M}{ \frac{f( \X )}{1 - 2 \cdot f( \X )} }. $$
Noting that the second term above does not depend on the scorer $s$, the result follows.
\end{proof}

\begin{proof}[Proof of Corollary \ref{corr:bayes-opt-same}]
By Proposition \ref{prop:risk-weighting},
\begin{align*}
    \underset{s}{\argmin} R( s; \D, \ell ) &= \underset{s}{\argmin} R^{\weight( w )}( s; \DCont, \ell )  \\
    &= \underset{s}{\argmin} R( s; \DCont, \ell ),
\end{align*}
where the second line is because weighting does {not} affect the Bayes-optimal scorers for a risk.
(Note that by definition, the weighting factor $w( x ) = (1 - 2 \cdot f( x ))^{-1} \geq 1$, and so no term is suppressed after weighting.)
\end{proof}

\begin{proof}[Alternate proof of Corollary \ref{corr:bayes-opt-same}]
If a loss $\ell$ satisfies Equation \ref{eqn:loss-sum-constant}, its conditional risk is
$$ L( \eta, v ) = (2 \cdot \eta - 1 ) \cdot \ell_1( v ) + C \cdot (1 - \eta). $$
Thus, the pointwise minimiser of the conditional risk is
\begin{align*}
     \underset{v}{\argmin} L( \eta, v ) &= \underset{v}{\argmin} (2 \cdot \eta - 1 ) \cdot \ell_1( v ) \\
     &= \underset{v}{\argmin} \begin{cases} \ell_1( v ) & \text{ if } \eta > \nicefrac[]{1}{2} \\ -\ell_1( v ) & \text{ if } \eta < \nicefrac[]{1}{2}, \end{cases}
 \end{align*}
implying a Bayes-optimal scorer of
$$ ( \forall x \in \XCal ) \, s^*( x ) = \underset{v}{\argmin} \begin{cases} \ell_1( v ) & \text{ if } \eta( x ) > \nicefrac[]{1}{2} \\ -\ell_1( v ) & \text{ if } \eta( x ) < \nicefrac[]{1}{2}. \end{cases} $$
Now we recall for the IDN model, $\eta( x ) > \nicefrac[]{1}{2} \iff \etaCont( x ) > \nicefrac[]{1}{2}$.
Thus, the two cases in the above scorer are the same for the clean and corrupted distributions.
It follows that the Bayes-optimal scorer is retained. 
\end{proof}

\begin{proof}[Proof of Proposition \ref{prop:regret-bound}]
Let $s^{*} \in \underset{s}{\argmin} R( s; \D, \ell )$.
By definition,
\begin{align*}
    \reg( s; \D, \ell ) &= R( s; \D, \ell ) - R( s^{*}; \D, \ell ) \\
    &= R^{\weight( w )}( s; \DCont, \ell ) - R^{\weight( w )}( s^{*}; \DCont, \ell ) \text{ by Proposition \ref{prop:risk-weighting} } \\
    &= \Expectation{\X \sim M}{ \frac{1}{1 - 2 \cdot \rho( \X )} \cdot ( L( \etaCont( \X ), s( \X ) ) - L( \etaCont( \X ), s^{*}( \X ) ) ) } \\
    &\leq \frac{1}{1 - 2 \cdot \rho_{\mathrm{max}}} \Expectation{\X \sim M}{ L( \etaCont( \X ), s( \X ) ) - L( \etaCont( \X ), s^{*}( \X ) ) } \text{ by assumption on } \rho \\
    &= \frac{1}{1 - 2 \cdot \rho_{\mathrm{max}}} \cdot ( R( s; \DCont, \ell ) - R( s^{*}; \DCont, \ell ) ) \\
    &= \frac{1}{1 - 2 \cdot \rho_{\mathrm{max}}} \cdot \reg( s; \DCont, \ell ),
\end{align*}
where the last line is since by Corollary \ref{corr:bayes-opt-same},
we know that $s^* \in \underset{s}{\argmin} R( s; \DCont, \ell )$ also.
(Note that for the inequality step above, we can guarantee $L( \etaCont( x ), s( x ) ) \geq L( \etaCont( x ), s^{*}( x ) )$ for every $x \in \XCal$ because $s^* \in \underset{s}{\argmin} R( s; \DCont, \ell )$, and so we do not have to worry about the direction of the inequality.)

To get the parameterised bound, suppose
$w( x ) = (1 - 2\cdot \rho(x))^{-1}$, and $r( x )$ is the conditional regret $L( \etaCont( x ), s( x ) ) - L( \etaCont( x ), s^*( x ) )$.
Then for any $\alpha \in [0, 1]$ the regret can be rewritten
\begin{align*}
      \reg( s; \D, \ell ) &= \Expectation{\X \sim M}{ w( \X ) \cdot r( \X ) } \\
      &= \int_{\XCal} m( x ) \cdot w( x ) \cdot r( x ) \, dx \\
      &= \int_{\XCal} m( x )^{\alpha} \cdot w( x ) \cdot m( x )^{1 - \alpha} \cdot r( x ) \, dx \\
      &= M \cdot R \cdot \int_{\XCal} m( x )^{\alpha} \cdot \frac{w( x )}{M} \cdot m( x )^{1 - \alpha} \cdot \frac{r( x )}{R} \, dx \text{ for } M = \max_x w( x ), R = \max_x r( x ) \\
      &\leq M^{1 - \alpha} \cdot R^{\alpha} \cdot \int_{\XCal} \left( m( x ) \cdot w( x ) \right)^\alpha \cdot ( m( x ) \cdot r( x ) )^{1 - \alpha} \, dx  \text{ since } x \leq x^\alpha \text{ for } \alpha \in [0, 1] \\
      &\leq M^{1 - \alpha} \cdot R^{\alpha} \cdot \left( \Expectation{\X \sim M}{ w( \X ) } \right)^\alpha \cdot \left( \reg( s; \DCont, \ell ) \right)^{1 - \alpha},
\end{align*}
{where the last line is by H\"{o}lder's inequality\footnote{In its native form, this states that $\sum_{i} |x_i| \cdot |y_i| \leq (\sum_{i} |x_i|^{1/\alpha})^{\alpha} \cdot (\sum_{i} |y_i|^{1/(1-\alpha)})^{1-\alpha}$, so that $\sum_{i} |x_i|^{\alpha} \cdot |y_i|^{1 - \alpha} \leq (\sum_{i} |x_i|)^{\alpha} \cdot (\sum_{i} |y_i|)^{1-\alpha}$.}.}
The case $\alpha = 0$ gives the original bound of Proposition \ref{prop:regret-bound}.
\end{proof}

\begin{proof}[Proof of Proposition \ref{prop:clean-corrupt-threshold-idn}]
Plug in $\rho_{\pm 1} \equiv f$ into Proposition \ref{prop:clean-corrupt-threshold}.
\end{proof}

\begin{proof}[Proof of Proposition \ref{prop:eta-monotone}]
If $\eta( x ) < \eta( x ' )$, then certainly $s( x ) < s( x' )$ since $s$ is order preserving for $\eta$ by BCN-admissibility Condition (a).
Thus, by Lemma \ref{lemm:eta-diff-bound}, 
$$ ( \forall x, x' \in \XCal ) \, \etaCont( x ) - \etaCont( x' ) \leq \max( 1 - \rho_{-1}( x ) - \rho_{1}( x ), 1 - \rho_{-1}( x' ) - \rho_{1}( x' ) ) \cdot ( \eta( x ) - \eta( x' ) ). $$
By the total noise assumption (Assumption \ref{ass:total-noise}), $1 - \rho_{-1}( x ) - \rho_{1}( x ) > 0$ for every $x$.
Since $\eta( x ) - \eta( x' ) < 0$ by assumption, we conclude that $\etaCont( x ) - \etaCont( x' ) < 0$.
\end{proof}

\begin{proof}[Proof of Proposition \ref{prop:auc-regret}]
From \citet{Clemencon:2008}, \citet[Theorem 11]{Agarwal:2014},
$$ \reg_{\mathrm{AUC}}( s; \D ) = \frac{1}{2 \cdot \pi \cdot (1 - \pi)} \cdot \Expectation{\X \sim M, \X' \sim M}{ | \eta( \X ) - \eta( \X' ) | \cdot \mathbb{I}( \eta( \X ) - \eta( \X' ), s( \X ) - s( \X' ) ) } $$
where
$$ \mathbb{I}( \Delta \eta, \Delta s ) = \indicator{ \Delta \eta \cdot \Delta s < 0 } + \nicefrac[]{1}{2} \cdot \indicator{ \Delta s = 0 }. $$

By Proposition \ref{prop:eta-monotone}, for this noise model,
$$ \eta( x ) \neq \eta( x' ) \implies \sign( \eta( x ) - \eta( x' ) ) = \sign( \etaCont( x ) - \etaCont( x' ) ), $$
Thus, in this case, $\sign( \Delta \eta ) = \sign( \Delta \etaCont )$, and so $\mathbb{I}( \Delta \eta, \Delta s ) = \mathbb{I}( \Delta \etaCont, \Delta s )$.
When $\eta( x ) = \eta( x' )$, however, there is no guarantee on the relative values of $\etaCont( x )$ and $\etaCont( x' )$.
But if $\Delta \eta = 0$, then the first term in $\mathbb{I}$ above is necessarily zero, while that for $\Delta \etaCont$ can only be $\geq 0$.
Thus, in general we have
$$ \mathbb{I}( \Delta \eta, \Delta s ) \leq \mathbb{I}( \Delta \etaCont, \Delta s ), $$
and so
$$ \reg_{\mathrm{AUC}}( s; \D ) \leq \frac{1}{2 \cdot \pi \cdot (1 - \pi)} \cdot \Expectation{\X \sim M, \X' \sim M}{ | \eta( \X ) - \eta( \X' ) | \cdot \mathbb{I}( \etaCont( \X ) - \etaCont( \X' ), s( \X ) - s( \X' ) ) }. $$

What remains then is the $| \eta( x ) - \eta( x' ) |$ term.
Now, by Lemma \ref{lemm:eta-diff-bound}, when $\eta( x ) \neq \eta( x' )$,
\begin{align*}
    \frac{\etaCont( x ) - \etaCont( x' )}{\eta( x ) - \eta( x' )} &\geq \max( 1 - \rho_{-1}( x ) - \rho_{1}( x ), 1 - \rho_{-1}( x' ) - \rho_{1}( x' ) ) \\
    &\geq 1 - 2 \cdot \rho_{\mathrm{max}}.
\end{align*}
If $\eta( x ) = \eta( x' )$, we trivially have $| \eta( x ) - \eta( x' ) | \leq | \etaCont( x ) - \etaCont( x' ) | \cdot (1 - 2 \cdot \rho_{\mathrm{max}})^{-1}$.
We conclude that
\begin{align*}
    \reg_{\mathrm{AUC}}( s; \D ) &\leq \frac{1}{2 \cdot \pi \cdot (1 - \pi)} \cdot \Expectation{\X \sim M, \X' \sim M}{ | \eta( \X ) - \eta( \X' ) | \cdot \mathbb{I}( \etaCont( \X ) - \etaCont( \X' ), s( \X ) - s( \X' ) ) } \\
    &\leq \frac{1}{2 \cdot \pi \cdot (1 - \pi)} \cdot \frac{1}{1 - 2 \cdot \rho_{\mathrm{max}}} \cdot \Expectation{\X \sim M, \X' \sim M}{ | \etaCont( \X ) - \etaCont( \X' ) | \cdot \mathbb{I}( \etaCont( \X ) - \etaCont( \X' ), s( \X ) - s( \X' ) ) } \\
    &= \frac{\piCont \cdot (1 - \piCont)}{\pi \cdot (1 - \pi)} \cdot \frac{1}{1 - 2 \cdot \rho_{\mathrm{max}}} \cdot \reg_{\mathrm{AUC}}( s; \DCont ).
\end{align*}
\end{proof}

\begin{proof}[Proof of Proposition \ref{prop:corrupt-sim}]
The form of $\etaCont$ follows from Equation \ref{eqn:eta-bcn} and Proposition \ref{lemm:corrupt-eta-general}.

Under Assumption \ref{ass:c}, the noise model $\SIN( \D, f_{-1}, f_{1} ) = \BCNPlus( \D, f_{-1}, f_{1}, s^* )$.
By Corollary \ref{corr:eta-cont-score-preserving},
$$ s^*( x ) < s^*( x' ) \implies \Contaminator{u}( s^*( x ) ) \leq \Contaminator{u}( s^*( x' ) ), $$
so that $\Contaminator{u}$ is a monotone function, and thus a valid GLM link.

Next, applying the triangle inequality to Lemma \ref{lemm:eta-diff},
and using $z = s( x ), z' = s( x' )$,
\begin{align*}
    | \etaCont( x ) - \etaCont( x' ) | &= | \Contaminator{u}( z ) - \Contaminator{u}( z' ) | \\
    &\leq | 1 - f_{-1}( z' ) - f_{1}( z' ) | \cdot | u( z ) - u( z' ) | + |f_{-1}( z ) - f_{-1}( z' )| \cdot |1 - u(z)| + |f_{1}( z ) - f_{1}( z' )| \cdot |u(z)| \\
    &\leq ( L + L_{-1} + L_{1} ) \cdot | z  - z' |,
\end{align*}
using the fact that $| 1 - f_{-1}( z' ) - f_{1}( z' ) | < 1$ by the total noise assumption (Assumption \ref{ass:total-noise}),
$|1 - u(z)| \leq 1$ and $|u(z)| \leq 1$ since $\mathrm{Im}( u ) = [ 0, 1 ]$, and
the Lipschitz assumptions on $u, f_{\pm 1}$.
It follows that $\Contaminator{u}$ is $( L + L_{-1} + L_{1} )$-Lipschitz.
\end{proof}

\begin{proof}[Proof of Proposition \ref{prop:isotron-consistency}]
By Proposition \ref{prop:corrupt-sim}, $\etaCont \in \SIM( L + L_2 + L_3, W )$.
Thus, as a member of the SIM family, it is suitable for estimation using SLIsotron.

Proposition \ref{prop:isotron} implies that one can always choose an iteration of SLIsotron with low regret.
Let $\hat{\etaCont}_{\mathsf{S}, t}$ denote the estimate produced by SLIsotron at iteration $t$.
If in an abuse of notation we let $\hat{\etaCont}_{\mathsf{S}}$ denote the estimate $\hat{\etaCont}_{\mathsf{S}, t^*}$, where $t^*$ is an appropriately determined iteration,
then we have that $\reg( \hat{\etaCont}_{\mathsf{S}}; \D, \ellSQ ) \stackrel{\Pr}{\to} 0$.

For AUC consistency, standard surrogate regret bounds \citep{Agarwal:2014} imply that
$$ \reg_{\mathrm{AUC}}( \hat{\etaCont}_{\SCont}; \DCont ) \leq \frac{1}{2 \cdot \piCont \cdot (1 - \piCont)} \cdot \sqrt{\reg( \hat{\etaCont}_{\SCont}; \DCont, \ellSQ )}. $$
By Proposition \ref{prop:auc-regret}, we conclude that
\begin{align*}
    \reg_{\mathrm{AUC}}( \hat{\etaCont}_{\SCont}; \D ) &\leq \frac{\piCont \cdot (1 - \piCont)}{\pi \cdot (1 - \pi)} \cdot \frac{1}{1 - 2 \cdot \rho_{\mathrm{max}}} \cdot \reg_{\mathrm{AUC}}( \hat{\etaCont}_{\SCont}; \DCont ) \\
    &\leq \frac{1}{2 \cdot \pi \cdot (1 - \pi)} \cdot \frac{1}{1 - 2 \cdot \rho_{\mathrm{max}}} \cdot \sqrt{\reg( \hat{\etaCont}_{\SCont}; \DCont, \ellSQ )}.
\end{align*}
The Isotron guarantee implies the RHS tends to 0 with sufficiently many samples.
Thus, $\reg_{\mathrm{AUC}}( \hat{\etaCont}_{\SCont}; \D ) \to 0$.

For classification consistency, standard surrogate regret bounds \citep{Zhang:2004,Bartlett:2006,Reid:2009} imply that we can bound the 0-1 regret in terms of the square loss regret:
$$ \reg( 2 \hat{\etaCont}_{\SCont} - 1; \DCont, \ellZO ) \leq \sqrt{\reg( \hat{\etaCont}_{\SCont}; \DCont, \ellSQ )}. $$
By Proposition \ref{prop:regret-bound}, for symmetric noise, thresholding our estimate of $\etaCont$ around $\nicefrac[]{1}{2}$ will be consistent wrt the clean distribution:
\begin{align*}
    \reg( c_{\SCont}; \D, \ellZO ) &= \reg( 2 \hat{\eta}_{\SCont} - 1; \D, \ellZO ) \\
    &\leq \frac{1}{1 - 2 \cdot \rho_{\mathrm{max}}} \cdot \reg( 2 \hat{\eta}_{\SCont} - 1; \DCont, \ellZO ) \\
    &= \frac{1}{1 - 2 \cdot \rho_{\mathrm{max}}} \cdot \sqrt{\reg( \hat{\etaCont}_{\mathsf{S}}; \DCont, \ellSQ )}.
\end{align*}
The Isotron guarantee implies the RHS tends to 0 with sufficiently many samples.
Thus, in the case of symmetric $\BCNPlus$ noise, thresholding $\etaCont$ around $\nicefrac[]{1}{2}$ will be consistent wrt the clean distribution.
\end{proof}


\section{Examples of the SIM family}
\label{app:sim-family}

Two simple examples of the SIM family are presented below.
The first was established in \citet{Kalai:2009}.

\begin{example}
Suppose that $\D$ corresponds to a concept that is linearly separable with margin $\gamma > 0$ \ie
$$ \eta( x ) = \indicator{ \langle w^*, x \rangle > 0 }, $$ 
and
$$ \Pr( \{ (x, y) \mid y \cdot \langle w^*, x \rangle < \gamma \} ) = 0. $$
Then, $ \eta \in \SIM( (2\gamma)^{-1}, || w^* || ) $.
The reason is that we can equally think of $\eta$ as
$$ \eta( x ) = \uMar( \langle w^*, x \rangle ) $$
where
\begin{equation}
    \label{eqn:u-margin}
     \uMar( z ) = \begin{cases} 1 & \text{ if } z > \gamma \\ \frac{z + \gamma}{2 \gamma} & \text{ if } z \in [-\gamma, +\gamma] \\ 0 & \text{ if } z < -\gamma. \end{cases}
 \end{equation}
The function $u$ is clearly $(2\gamma)^{-1}$-Lipschitz.
\end{example}

\begin{example} Suppose that $\D$ corresponds to a concept that can be modelled using logistic regression \ie
$$ \eta( x ) = \frac{1}{1 + e^{-\langle w^*, x \rangle}}. $$
Then, $ \eta \in \SIM( 1, || w^* || ) $.
\end{example}

\section{Special cases of the ILN model}
\label{app:iln-special-case}

Several special cases of the ILN model are of interest.
(Table \ref{tbl:noise-models} summarises.)

\subsection{Instance-independent noise models}

The following have been the focus of a vast literature.

\begin{definition}[Noise-free learning]
Suppose we have an ILN model $ \ILN( \D, \rho_{-1}, \rho_{1} ) $ where $\rho_{\pm 1} \equiv 0$.
Then, we have the standard problem of learning from (noise free) binary labels.
\end{definition}

\begin{definition}[SLN model]
Suppose we have an ILN model $ \ILN( \D, \rho_{-1}, \rho_{1} ) $ where $\rho_{\pm 1} \equiv \rho$ for some constant $\alpha < \frac{1}{2}$.
Then, we have the problem of learning with symmetric label noise (SLN learning), also known as the problem of learning with random classification noise (RCN learning) \citep{Long:2008, vanRooyen:2015}.
We will write the corresponding corrupted distribution as $\SLN( \D, \alpha )$.
\end{definition}

\begin{definition}[CCN model]
Suppose we have an ILN model $ \ILN( \D, \rho_{-1}, \rho_{1} ) $ where $\rho_1 \equiv \alpha, \rho_{-1} \equiv \beta$ for some constants $\alpha, \beta < 1$.
Then, we have the problem of learning with class-conditional label noise (CCN learning) \citep{Angluin:1988,Blum:1998,Scott:2013,Natarajan:2013}.
We will write the corresponding corrupted distribution as $\CCN( \D, \beta, \alpha )$.
\end{definition}

\subsection{Boundary-consistent noise models}

The following is a far-reaching generalisation of the above.

\begin{definition}[BCN model]
Suppose we have an ILN model $ \ILN( \D, \rho_{-1}, \rho_{1} ) $ where $\rho_y = f_y \circ s$ for some functions $f_{\pm 1} \colon \Real \to [0, 1]$, and a function $s \colon \XCal \to \Real$ such that:
\begin{enumerate}[(a)]
    \item $s$ is order preserving for $\eta$ \ie
$$ ( \forall x, x' \in \XCal ) \, \eta( x ) < \eta( x' ) \implies s( x ) < s( x' ), $$
or equivalently,
$$ ( \exists u \colon \Real \to [ 0, 1 ] \text{ monotone} ) \, \eta = u \circ s. $$
The $u$ above is not required to be \emph{strictly} monotone, so it may not be true that $s = v \circ \eta$ for some $v \colon [0, 1] \to \Real$;
as a simple example, suppose that $\eta( x ) = \indicator{ s( x ) > 0 }$.

    \item $f_{\pm 1}$ are non-decreasing on $(-\infty, u^{\dagger}( 1/2 )]$ and non-increasing on $[u^{\dagger}( 1/2 ), \infty)$, where
$$ u^{\dagger}( 1/2 ) = \sup_{z \in \Real} \left\{ z \colon u( z ) \leq \frac{1}{2} \right\} $$ 
is the generalised inverse of $u$ at $\frac{1}{2}$;
or more compactly, if $f_{\pm 1}$ are differentiable,
\begin{equation}
    \label{eqn:bcn-admissible}
     ( \forall z \in [0, 1] ) \, f_{\pm 1}'( z ) \cdot (z - u^{\dagger}( 1/2 )) \leq 0.
\end{equation}
\end{enumerate}

In the case where $\D$ is linearly separable, and $s$ is such that $u^{\dagger}(1/2) = 0$, such a model was considered\footnote{\citet{Du:2015} considers $\Pr( \Contaminator{\Y} \neq \Y \mid \X = x )$, which is precisely our $\rho_{\Y}( x )$.} in \citet{Du:2015},
where it was termed learning with boundary consistent noise (BCN learning).
(A similar model was studied in \citet{Bootkrajang:2016} from a probabilistic perspective.)
We borrow this terminology for the case of general $\D$.
We will write the corresponding corrupted distribution as $\BCN( \D, f_{-1}, f_{1}, s )$;
further, we will say that $( f_{-1}, f_{1}, s, \eta )$ are \emph{BCN-admissible} if they satisfy the conditions detailed above.
\end{definition}

The above model in turn has several special cases that are of interest.

\subsubsection{The probabilistically transformed noise model}

A simple noise model is where the noise is some monotone transformation of the underlying $\eta$.

\begin{definition}[PTN model]
\label{ex:gpcn}
Suppose we have an ILN model $ \ILN( \D, \rho_{-1}, \rho_{1} ) $ where $\rho_y = f_y \circ \eta$ for some functions $f_{\pm 1} \colon [0, 1] \to [0, 1]$ 
such that $(f_{-1}, f_{1}, \eta, \eta)$ are BCN-admissible.
In this model, labels are flipped with higher probability for those instances with high inherent uncertainty (\ie with $\eta$ values close to $\frac{1}{2}$).
We term this problem learning with probabilistically transformed noise (PTN learning).
We will write the corresponding corrupted distribution as $\PTN( \D, f_{-1}, f_{1} )$;
further, we will say that $( f_{-1}, f_{1} )$ are \emph{PTN-admissible} if they satisfy the conditions detailed above.
\end{definition} 

When $f_{\pm 1} \equiv f$ above, we flip labels with probability proportional to the distance of $\eta$ from $\nicefrac[]{1}{2}$.
In the general case, it is intuitive that we will need some conditions on $f_{\pm 1}$ to ensure order preservation.

\subsubsection{The Bylander model}

\citet{Bylander:1997} describes the following model, termed monotonic or probabilistically-consistent noise.
One has a distribution $\D$ that is linearly separable with margin $\gamma > 0$,
\ie $\eta( x ) = \uMar( \langle w^*, x \rangle )$ with $\uMar$ as per Equation \ref{eqn:u-margin}.
One observes samples from a distribution $\DCont$ that satisfies the following conditions:
\begin{align*}
   \langle w^*, x \rangle \geq \langle w^*, x' \rangle  &\implies \frac{\eta( x )}{1 - \eta( x )} \geq \frac{\eta( x' )}{1 - \eta( x' )} \\
   \langle w^*, x \rangle \geq -\langle w^*, x' \rangle  &\implies \frac{\eta( x )}{1 - \eta( x )} \geq \frac{1 - \eta( x' )}{\eta( x' )},
\end{align*}
or equivalently,
\begin{align*}
   \langle w^*, x \rangle \geq \langle w^*, x' \rangle  &\implies \eta( x ) \geq \eta( x' ) \\
   \langle w^*, x \rangle \geq -\langle w^*, x' \rangle  &\implies \eta( x ) \geq 1 - \eta( x' ).
\end{align*}
The contrapositive of these implications is
\begin{align*}
   \eta( x ) < \eta( x' ) &\implies \langle w^*, x \rangle < \langle w^*, x' \rangle \\
   \eta( x ) < 1 - \eta( x' ) &\implies \langle w^*, x \rangle < -\langle w^*, x' \rangle.
\end{align*}
The first of these implications means that $\eta( x ) = \phi( \langle w^*, x \rangle )$ for some non-decreasing $\phi$.
The second of these implications is satisfied if $\phi( -z ) = 1 - \phi( z )$.
We formalise this as follows.

\begin{definition}[BYLN model]
\label{ex:bylander}
Suppose we have an ILN model $ \ILN( \D, \rho_{-1}, \rho_{1} ) $ where $\rho_y = f_y \circ s$, where
$f_1 \equiv f_{-1} \equiv f$ such that
\begin{enumerate}[(a)]
    \item $(f, f, s, \eta)$ is BCN-admissible,

    \item $f$ is symmetric around $u^{\dagger}(1/2)$,
    \ie $ f( z ) = g( | z - u^{\dagger}(1/2) | )$ for some non-increasing function $g \colon \Real_+ \to \Real_+$.
\end{enumerate}
We term this model for general $\D$ as learning with {Bylander noise} (BYLN learning).
We will write the corresponding corrupted distribution as $\BYLN( \D, f, s )$;
further, we will say that $( f, s, \eta )$ are \emph{BYLN-admissible} if they satisfy the conditions detailed above.
\end{definition}

In the case where $\D$ is linearly separable, and $s$ is such that $u^{\dagger}(1/2) = 0$,
the BYLN model is as considered in \citet{Bylander:1997, Bylander:1998, Servedio:1999}.

\subsubsection{The $\BCNPlus$ model}

The $\BCNPlus$ model introduced in Definition \ref{defn:bcn-plus} is seen to be the BCN model augmented with an additional assumption.

\begin{assumption}
\label{ass:c}
The difference $\Delta( z ) = f_{1}( z ) - f_{-1}( z )$ between the positive and negative flip functions is non-increasing.
\end{assumption}

This assumption proves crucial in guaranteeing that $\etaCont$ is order-preserving for $\eta$ (see Appendix \ref{app:order-fails}).
It is trivially satisfied in special cases of the $\BCNPlus$ model.

\begin{example}
For the case of CCN noise $\CCN( \D, \rho_{-1}, \rho_{1} )$, the flip functions are constant, and so Assumption \ref{ass:c} is trivially satisfied.
\end{example}

\begin{example}
For the case of symmetric BCN noise $\BCN( \D, f, f, s )$, the difference between the flip functions is a constant, and so Assumption \ref{ass:c} is trivially satisfied.
\end{example}

\subsection{Instance-dependent model}

The final special case of ILN is a generic instance- (but not label-) dependent noise model, previously considered in \citet{Ghosh:2015}.

\begin{definition}[IDN model]
Suppose we have an ILN model $ \ILN( \D, \rho_{-1}, \rho_{1} ) $ where $\rho_{-1} \equiv \rho_1 \equiv f$ for some function $f \colon \XCal \to [0, 1/2)$.
We term this problem learning with instance-dependent noise (IDN learning).
We will write the corresponding corrupted distribution as $\IDN( \D, f )$.
\end{definition}

\subsection{Relation between the noise models}

The above noise models are related to each other as follows:
\begin{align*}
    \SLN( \D, \alpha ) &= \CCN( \D, \alpha, \alpha ) \\
    \CCN( \D, \beta, \alpha ) &= \PTN( \D, \beta \cdot \mathbb{1}, \alpha \cdot \mathbb{1} ) \\
    \PTN( \D, f_{-1}, f_{1} ) &= \BCN( \D, f_{-1}, f_{1}, \eta ) \\
    \BYLN( \D, f, s ) &= \BCN( \D, f, f, s ) \\
    \BCN( \D, f_{-1}, f_{1}, s ) &= \ILN( \D, f_{-1} \circ s, f_{1} \circ s ) \\
    \BCN( \D, f, f, s ) &= \IDN( \D, f \circ s ) \\
    \IDN( \D, f ) &= \ILN( \D, f, f ).
\end{align*}
Here, $\mathbb{1}$ refers to the function which is $1$ everywhere.

\begin{table}[!t]
    \centering
    \renewcommand{\arraystretch}{1.25}
    \begin{tabular}{@{}p{1.5in}lp{2.5in}@{}}
        \toprule
        \toprule
        \textbf{Noise model} & \textbf{Notation} & \textbf{Description} \\
        \toprule
        Instance- and label-dependent noise & $\ILN( \D, \rho_{-1}, \rho_{1} )$ & Flip probability function of instance and label \\
        Instance-dependent noise & $\IDN( \D, f )$ & Flip probability function of instance only \\
        Class-conditional noise & $\CCN( \D, \beta, \alpha )$ & Flip probability depends on label only \\
        Symmetric label noise & $\SLN( \D, \alpha )$ & Constant flip probability \\
        Boundary-conditional noise & $\BCN( \D, f_{-1}, f_{1}, s )$ & Flip probability function of \emph{score} on instance and label, where score is consistent with underlying class-probability function \\
        Bylander noise & $\BYLN( \D, f, s )$ & Flip probability function of \emph{score} on instance only, where score is consistent with underlying class-probability function \\
        Probabilistically transformed noise & $\PTN( \D, f_{-1}, f_{1} )$ & Flip probability function of underlying class-probability function and label \\
        \bottomrule
    \end{tabular}
    
    \caption{Summary of noise models.}
    \label{tbl:noise-models}
\end{table}

\begin{remark}
As seen above, the BCN model reduces to the PTN model when $u$ is invertible.
Note however that when $u$ is not invertible, the BCN model is more powerful than the PTN model.
For example, if $\D$ is separable with a margin, then under the PTN model, all deterministically positive instances are flipped with some probability, and similarly all deterministically negative instances.
However, under the BCN model, instances closer to the optimal decision boundary, regardless of their label, will have a higher chance of being flipped.
\end{remark}

\section{Special cases of $\DCont$}
\label{app:additional-properties}

We list the components of $\DCont$ in some special cases.

\begin{example}
For the class-conditional noise model $\CCN( \D, \beta, \alpha )$, we have
\begin{align}
    \label{eqn:ccn-eta}
     \etaCont( x ) &= (1 - \alpha - \beta ) \cdot \eta( x ) + \beta \\
     \nonumber \piCont &= \pi \cdot (1 - \alpha - \beta) + \beta \\
     \nonumber \PCont( x ) &= \piCont^{-1} \cdot \left( (1 - \alpha) \cdot \pi \cdot P( x ) + \beta \cdot (1 - \pi) \cdot Q( x ) \right) \\
     \nonumber \QCont( x ) &= (1 - \piCont)^{-1} \cdot \left( \alpha \cdot \pi \cdot P( x ) + (1 - \beta) \cdot (1 - \pi) \cdot Q( x ) \right).
 \end{align}
 This is in agreement with \citet[Lemma 7]{Natarajan:2013}, \citet[Appendix C]{Menon:2015}.
\end{example}

\begin{example}
For the boundary-conditional noise model $\BCN( \D, f_{-1}, f_{1}, s )$,
with $\eta = u \circ s$ for some non-decreasing $u$,
we have
\begin{equation}
\label{eqn:eta-bcn}
\begin{aligned}    
     \etaCont( x ) &= \Contaminator{u}( s( x ) ) \text{ where } \\
     \Contaminator{u}( z ) &= ( 1 - f_1( z ) ) \cdot u( z ) + f_{-1}( z ) \cdot ( 1 - u ( z ) ) \\
     &= ( 1 - f_1( z ) - f_{-1}( z ) ) \cdot u( z ) + f_{-1}( z ).
\end{aligned}
\end{equation}
The monotonicity of $\Contaminator{u}$ is studied in Lemma \ref{lemm:eta-diff-bound}.
\end{example}

\begin{example}
For the PTN model $\PTN( \D, f_{-1}, f_{1} )$, we have
By Proposition \ref{lemm:corrupt-eta-general}, and the fact that $\rho_y = g \circ \eta$ for the PTN model,
\begin{equation}
\label{eqn:eta-ptn}
\begin{aligned}    
    \etaCont( x ) &= \varphi( \eta( x ) ) \\
    \varphi( z ) = (1 - f_{-1}( z ) - f_{1}( z )) \cdot z + f_{-1}( z ).
\end{aligned}
\end{equation}
\end{example}

\begin{example}
For the instance-dependent noise model $\IDN( \D, f )$, we have
\begin{align*}
     \etaCont( x ) &= (1 - 2 \cdot f( x ) ) \cdot \eta( x ) + f( x ) \\
     \piCont &= \pi + \Expectation{\X \sim M}{ f( \X ) \cdot ( 1 - 2 \cdot \eta( \X ) ) } \\
     \PCont( x ) &= \piCont^{-1} \cdot \left( (1 - f( x )) \cdot \pi \cdot P( x ) + f( x ) \cdot (1 - \pi) \cdot Q( x ) \right) \\
     \QCont( x ) &= (1 - \piCont)^{-1} \cdot \left( f( x ) \cdot \pi \cdot P( x ) + (1 - f( x )) \cdot (1 - \pi) \cdot Q( x ) \right).
 \end{align*}
For the class-conditionals, we can equally write
\begin{align*}
     P( x ) &= (1 - 2 \cdot f( x ))^{-1} \cdot \pi^{-1} \cdot \left( (1 - f( x )) \cdot \piCont \cdot \PCont( x ) - f( x ) \cdot (1 - \piCont) \cdot \QCont( x ) \right) \\
     Q( x ) &= (1 - 2 \cdot f( x ))^{-1} \cdot (1 - \pi)^{-1} \cdot \left( -f( x ) \cdot \piCont \cdot \PCont( x ) + (1 - f( x )) \cdot (1 - \piCont) \cdot \QCont( x ) \right).
 \end{align*}
\end{example}

\section{Boundary consistent noise and flip probabilities}

Given an instance $x \in \XCal$, let $F( x )$ denote the probability that the instance has its label flipped.
It is easy to check that
\begin{align*}
    F( x ) &= \Pr( \Y \neq \YCont \mid \X = x ) \\
    &= \Pr( \Y \neq \YCont \mid \Y = 1, \X = x ) \cdot \Pr( \Y = 1 \mid \X = x ) + \Pr( \Y \neq \YCont \mid \Y = -1, \X = x ) \cdot \Pr( \Y = -1 \mid \X = x ) \\
    &= \rho_1( x ) \cdot \eta( x ) + \rho_{-1}( x ) \cdot (1 - \eta( x )) \\
    &= (\rho_1( x ) - \rho_{-1}( x )) \cdot \eta( x ) + \rho_{-1}( x ).
\end{align*}

\subsection{Guaranteeing maximisation at $\frac{1}{2}$ for BCN model}

In the case of boundary consistent noise,
$$ F( x ) = \varphi( s( x ) ) $$
where
$$ \varphi( z ) = (f_1( z ) - f_{-1}( z )) \cdot u( z ) + f_{-1}( z ). $$
Suppose we want $F$ to be increasing when $\eta < \nicefrac[]{1}{2}$, and decreasing otherwise.
Observe that
\begin{align*}
     \varphi'( z ) &= (f_1( z ) - f_{-1}( z )) \cdot u'( z ) + (f'_1( z ) - f'_{-1}( z )) \cdot u( z ) + f'_{-1}( z ) \\
     &= (f_1( z ) - f_{-1}( z )) \cdot u'( z ) + f'_1( z ) \cdot u( z ) + f'_{-1}( z ) \cdot (1 - u( z )).
 \end{align*}
When $u( z ) < \frac{1}{2}$, the second and third terms are guaranteed to be positive (by Condition (b) of BCN-admissibility).
Since $u'( z ) > 0$, for the first term to be positive we need $\Delta( z ) \geq 0$.
Similarly, when $u( z ) > \frac{1}{2}$, the second and third terms are guaranteed to be negative;
for the first term to be negative we need $\Delta( z ) \leq 0$.

Thus, a sufficient condition for $F$ to be maximised when $\eta = \frac{1}{2}$ is for
\begin{equation}
    \label{eqn:delta}
    \Delta( z ) \cdot (2 \cdot u( z ) - 1) \leq 0.
\end{equation}

\subsection{Relation to Assumption \ref{ass:c}}

Note that above, we do \emph{not} require Assumption \ref{ass:c} (\ie $\Delta( z )$ is decreasing).
Indeed, Assumption \ref{ass:c} by itself does not guarantee that $F$ is maximised when $\eta = \frac{1}{2}$.
As a simple example, for the CCN model,
$$ F( x ) = ( \alpha - \beta ) \cdot \eta( x ) + \beta. $$
Evidently, this is maximised at either $\eta( x ) = 0$ or $\eta( x ) = 1$, depending on whether $\alpha > \beta$ or not.

On the other hand, $\Delta( z )$ satisfying Equation \ref{eqn:delta} by itself does not guarantee that $\etaCont$ is order preserving for $\eta$.
Consider for example a case where $f_1( z ) = \indicator{z \leq 0} \cdot \frac{1}{2} \cdot e^z$, $f_{-1}( z ) \equiv 0$, and $\eta( x ) = 1/(1 + \exp(-s(x)))$.
Then, $\etaCont$ will not be order preserving for $\eta$.

If $\Delta( z )$ satisfies both Equation \ref{eqn:delta} and Assumption \ref{ass:c}, then we will have that $F$ is maximised when $\eta = \frac{1}{2}$, and also that $\etaCont$ is order preserving for $\eta$.

\section{Label swapping and Assumption \ref{ass:c}}

Assumption \ref{ass:c} implies an asymmetry in the treatment of positive and negative labels.
This is at first glance surprising, since intuitively we would expect our results to hold even if we swap the labels.
In particular, suppose we have some $\D = (M, \eta)$ with a $\BCNPlus( \D, f_{-1}, f_{1}, s )$ noise model.
Then,
$$ 1 - \etaCont( x ) = f_{1}( s( x ) ) \cdot \eta( x ) + (1 - f_{-1}( s( x ) )) \cdot (1 - \eta( x )). $$
Now consider $\D' = (M, 1 - \eta)$ so that the positive and negative labels are swapped.
Then, it is not hard to see that for a $\BCNPlus( \D', f_{1}, f_{-1}, -s )$ noise model,
$$ \etaCont'( x ) = (1 - f_{-1}( -s( x ) )) \cdot (1 - \eta( x )) + f_{1}( -s( x ) ) \cdot \eta( x ). $$
Therefore, if the flip functions $f_{\pm 1}$ are even (\ie symmetric around the origin), we have
$$ \etaCont'( x ) = 1 - \etaCont( x ). $$
So,
\begin{align*}
    \eta'( x ) < \eta'( x' ) &\iff \eta( x ) > \eta( x' ) \\
    &\implies \etaCont( x ) > \etaCont( x' ) \\
    &\implies \etaCont'( x ) < \etaCont'( x' ).
\end{align*}
Thus, order preservation is retained.
This may seem peculiar since for $\D'$, we have the opposite of Assumption \ref{ass:c} holding.
But note that
$$ \etaCont'( x ) = \varphi( -s( x ) ) $$
where
\begin{align*}
     \varphi( z ) &= (1 - f_{-1}( z )) \cdot (1 - u( z )) + f_1( z ) \cdot u( z ) \\
     &= 1 - \left( f_{-1}( z ) \cdot (1 - u( z )) + (1 - f_1( z )) \cdot u( z ) \right) \\
     &= 1 - \left( f_{-1}( z ) \cdot v( z ) + (1 - f_1( z )) \cdot (1 - v( z )) \right)
\end{align*}
where $v( z ) = 1 - u( z )$.
For $\varphi( z )$ to be non-decreasing, the second term above must be non-increasing.
This term is precisely that arising from the standard BCN model, but with a link function $v$ that is non-increasing, and with flip functions satisfying the opposite of Assumption \ref{ass:c}.
Thus, it is not hard to see that we can guarantee the opposite of the standard BCN model, so that the term is non-increasing.

\section{The generalised loss object under the CCN model}
\label{app:generalised-loss}

For the class-conditional noise model $\CCN( \D, \rho_{-1}, \rho_{1} )$,
the generalised loss of Proposition \ref{prop:general-mube} is simply
\begin{align*}
    \tilde{\ell}_1( s, x ) &= w^{-1} \cdot \left( (1 - \rho_{-1} \cdot \ell_1( s( x ) ) - \rho_{1} \cdot \ell_{-1}( s( x ) ) \right) \\
    \tilde{\ell}_{-1}( s, x ) &= w^{-1} \cdot \left( -\rho_{-1} \cdot \ell_1( s( x ) ) + (1 - \rho_{1} ) \cdot \ell_{-1}( s( x ) ) \right)
\end{align*}
where $w = 1 - \rho_{-1} - \rho_{1}$.
The dependence on $x$ is only via the correpsonding $s( x )$ value.
Thus, we may equally consider the noise-corrected loss
\begin{align*}
    \Contaminator{\ell}_1( v ) &= w^{-1} \cdot \left( (1 - \rho_{-1} \cdot \ell_1( v ) - \rho_{1} \cdot \ell_{-1}( v ) \right) \\
    \Contaminator{\ell}_{-1}( v ) &= w^{-1} \cdot \left( -\rho_{-1} \cdot \ell_1( v ) + (1 - \rho_{1} ) \cdot \ell_{-1}( v ) \right),
\end{align*}
with Proposition \ref{prop:general-mube} then reducing to
$$ R( s; \D, \ell ) = R( s; \DCont, \Contaminator{\ell} ), $$
as shown in \citet[Lemma 1]{Natarajan:2013}, who termed the approached of minimising $\Contaminator{\ell}$ as the ``method of unbiased estimators''.

\begin{remark}
\citet[Lemma 1]{Natarajan:2013} was generalised in a different direction by \citet{vanRooyen:2015b}, who considered problems with general label spaces.
The noise model in \citet{vanRooyen:2015b} is still instance independent, unlike Proposition \ref{prop:general-mube}.
\end{remark}

\section{Failure of order preservation under $\etaCont$}
\label{app:order-fails}

We illustrate that for noise models other than $\BCNPlus$, order preservation under $\etaCont$ is not guaranteed.

\subsection{Failure of order preservation for the BCN model}

Order preservation is not guaranteed for the BCN model without Condition (c) of the $\BCNPlus$ model.



\begin{example}
Suppose $f_{1}( z ) \equiv 0$, $f_{-1}( z ) = a \cdot \indicator{z \leq 0}$ for some $a < 1$,
and $s$ is such that $\eta( x ) = \frac{1}{1 + e^{-s( x )}}$.
Certainly $( f_{-1}, f_{1}, s )$ is BCN-admissible.
It is easy to check that
$$ \etaCont( x ) = \varphi( s( x ) ) $$
where
\begin{align*}
     \varphi( z ) &= \left( 1 - a \cdot \indicator{z \leq 0} \right) \cdot \frac{e^z}{1 + e^z} + a \cdot \indicator{z \leq 0} \\
     &= \begin{cases} (1 - a) \cdot \frac{e^z}{1 + e^z} + a & \text{ if } z \leq 0 \\ \frac{e^z}{1 + e^z} & \text{ if } z > 0, \end{cases}
 \end{align*}
 which is easily checked to not be monotone in $z$.
\end{example}

The difference $\Delta( z ) = f_1( z ) - f_{-1}( z )$ above is non-decreasing.
Swapping the flip functions thus makes the function non-increasing, satisfying Assumption \ref{ass:c}.
We can confirm that in this case, $\etaCont$ will indeed be order-preserving for $\eta$.

\begin{example}
Suppose $f_{-1}( z ) \equiv 0$, $f_{1}( z ) = a \cdot \indicator{z \leq 0}$ for some $a < 1$,
and $s$ is such that $\eta( x ) = \frac{1}{1 + e^{-s( x )}}$.
Certainly $( f_{-1}, f_{1}, s )$ is BCN-admissible.
It is easy to check that
$$ \etaCont( x ) = \varphi( s( x ) ) $$
where
\begin{align*}
     \varphi( z ) &= \left( 1 - a \cdot \indicator{z \leq 0} \right) \cdot \frac{e^z}{1 + e^z} \\
     &= \begin{cases} a \cdot \frac{e^z}{1 + e^z} & \text{ if } z \leq 0 \\ \frac{e^z}{1 + e^z} & \text{ if } z > 0, \end{cases}
 \end{align*}
 which is easily checked to be monotone in $z$.
\end{example}

\subsection{Failure of order preservation for the IDN model}

For the IDN model, order preservation will not be guaranteed in general.
Consider the simple case where $f( x ) = \frac{1}{2} \eta( x )$.
This means that there is more noise for positive instances.
Then, we have
\begin{align*}
   \etaCont( x ) &= (1 - \eta( x )) \cdot \eta( x ) + \frac{1}{2} \cdot \eta( x ) \\
   &= \eta( x ) \cdot \left( \frac{3}{2} - \eta( x ) \right).
 \end{align*}
This will not be order preserving for $\eta$, since $\varphi( z ) = z \cdot (2 - z)$ is not monotone on $[0,1]$.

\section{On Bayes-optimal scorers coinciding on clean and corrupted distributions}
\label{app:bayes-opt}

Corollary \ref{corr:bayes-opt-same} is a statement about the minimisers when using all measurable scorers $\Real^{\XCal}$.
When using \eg linear scorers, one does not have the same equivalence in general, unless the Bayes-optimal scorer happens to lie in our chosen class.
For example, with the unhinged loss and kernelised linear scorers $\langle w, \Phi( x ) \rangle_{\HCal}$ for some RKHS $\HCal$, under the IDN model we have optimal weight
\begin{align*}
    w^* &= \Expectation{( \X, \YCont ) \sim \DCont}{ \Y \cdot \Phi( \X ) } \\
    &= \Expectation{\X \sim M}{ \Phi( \X ) \cdot ( 2 \cdot \etaCont( \X ) - 1 ) } \\
    &= \Expectation{\X \sim M}{ ( 1 - 2 \cdot f( \X ) ) \cdot \Phi( \X ) \cdot ( 2 \cdot \eta( \X ) - 1 ) },
\end{align*}
which possesses an additional weighting term compared to the optimal weight on $\D$.
Nonetheless, we can expect the scores resulting from this solution to have the correct sign for classification.
The score on an instance $x' \in \XCal$ is
$$ s^*( x' ) = \Expectation{\X \sim M}{ ( 1 - 2 \cdot f( \X ) ) \cdot k( \X, x' ) \cdot ( 2 \cdot \eta( \X ) - 1 ) } $$
for kernel function $k( x, x' ) = \langle \Phi( x ), \Phi( x' ) \rangle_{\HCal}$.
If\footnote{This is not actually a valid kernel for an RKHS, since delta functions are not square integrable.} $k( x, x' ) = \delta_{x}( x')$, this would reduce to $(1 - 2 \cdot f( x' ) ) \cdot (2 \cdot \eta( x') - 1) \cdot m( x' )$, which has the same sign as $2 \cdot \eta( x') - 1$.

\section{Proof of Proposition \ref{prop:regret-bound} specialised to 0-1 loss}
\label{app:regret-01}

The regret for 0-1 loss is \citep[Theorem 2.2]{Devroye:1996}, \citep[Lemma 8]{Reid:2009}
\begin{equation}
    \label{eqn:regret-01}
    \reg( s; \D, \ellZO ) = \Expectation{\X \sim M}{ \left| \eta( \X ) - \frac{1}{2} \right| \cdot \indicator{ (2\eta( \X ) - 1) \cdot s( \X ) < 0 } }.
\end{equation}
The following shows that for 0-1 loss and label-independent noise, there is a simple relationship between the regrets on the clean and corrupted distributions.
A key ingredient is the following.

\begin{proposition}
\label{prop:noisy-clean-01}
Pick any distribution $\D$.
Suppose that $\DCont = \IDN( \D, f )$
for admissible $f \colon \XCal \to [0, 1/2)$.
Then,
$$ ( \forall x \in \XCal ) \, \eta( x ) - \frac{1}{2} = \frac{1}{1 - 2 \cdot f( x )} \cdot \left( \etaCont( x ) - \frac{1}{2} \right). $$
\end{proposition}

\begin{proof}
By Proposition \ref{lemm:corrupt-eta-general},
$$ \etaCont( x ) = ( 1 - 2 \cdot f( x ) ) \cdot \eta( x ) + f( x ). $$
Thus,
\begin{align*}
    \etaCont( x ) - \frac{1}{2} &= \eta( x ) - \frac{1}{2} + f( x ) \cdot (1 - 2 \cdot \eta( x )) \\
    &= \eta( x ) - \frac{1}{2} + 2 \cdot f( x ) \cdot \left( \frac{1}{2} - \eta( x ) \right) \\
    &= \left( \eta( x ) - \frac{1}{2} \right) \cdot ( 1 - 2 \cdot f( x ) ).
\end{align*}
Thus, since $f( x ) \neq \frac{1}{2}$,
$$ \eta( x ) - \frac{1}{2} = \frac{1}{1 - 2 \cdot f( x )} \cdot \left( \etaCont( x ) - \frac{1}{2} \right). $$
\end{proof}

The above required that no instance has label flipped with probability $\frac{1}{2}$, which is a mild and intuitive condition.
If we further assume that the flip probability for every instance is \emph{less than} $\frac{1}{2}$,
this simple relationship implies the Bayes-optimal classifier is unaffected.

\begin{corollary}
\label{corr:noisy-clean-01}
Pick any distribution $\D$.
Suppose that $\DCont = \IDN( \D, f )$
for admissible $f \colon \XCal \to [0, 1/2).$
Then,
$$ ( \forall x \in \XCal ) \, \eta( x ) > \frac{1}{2} \iff \etaCont( x ) > \frac{1}{2} $$
and so
$$ \underset{s}{\argmin} R( s; \D, \ellZO ) = \underset{s}{\argmin} R( s; \DCont, \ellZO ). $$
\end{corollary}

\begin{proof}[Proof of Corollary \ref{corr:noisy-clean-01}]
By Proposition \ref{prop:noisy-clean-01}, if $f( x ) < \frac{1}{2}$ for every $x$, so that $1 - 2 \cdot f( x ) > 0$,
the two class-probability functions have the same sign around $\frac{1}{2}$.
Thus, $\eta( x ) > \frac{1}{2} \iff \etaCont( x ) > \frac{1}{2}$.

Alternately, simply plug in $t = \frac{1}{2}$ to Proposition \ref{prop:clean-corrupt-threshold}.
\end{proof}

We are now in a position to provide the regret bound.

\begin{proposition}
Pick any distribution $\D$.
Suppose that $\DCont = \IDN( \D, f )$
for admissible $f \colon \XCal \to [0, 1]$ such that
$$ ( \forall x \in \XCal ) \, f( x ) \leq \rho_{\mathrm{max}} < \frac{1}{2}. $$
Then, for any scorer $\scorer$,
$$ \reg( s; \D, \ellZO ) \leq \frac{1}{1 - 2 \cdot \rho_{\mathrm{max}}} \cdot \reg( s; \DCont, \ellZO ). $$
\end{proposition}

\begin{proof}
If $f( x ) \leq \rho_{\mathrm{max}}$, then $1 - 2 \cdot f( x ) \geq 1 - 2 \cdot \rho_{\mathrm{max}}$, and so by Proposition \ref{prop:noisy-clean-01},
$$ \eta( x ) - \frac{1}{2} \leq \frac{1}{1 - 2 \cdot \rho_{\mathrm{max}}} \cdot \left( \etaCont( x ) - \frac{1}{2} \right). $$
Now, by Equation \ref{eqn:regret-01}, for any scorer $s$,
\begin{align*}
    \reg( s; \D, \ellZO ) &= \Expectation{\X \sim M}{ \left| \eta( \X ) - \frac{1}{2} \right| \cdot \indicator{ (2\eta( \X ) - 1) \cdot s( \X ) < 0 } } \\
    &\leq \frac{1}{1 - 2 \cdot \rho_{\mathrm{max}}} \cdot \Expectation{\X \sim M}{ \left| \etaCont( \X ) - \frac{1}{2} \right| \cdot \indicator{ (2\eta( \X ) - 1) \cdot s( \X ) < 0 } } \\
    &= \frac{1}{1 - 2 \cdot \rho_{\mathrm{max}}} \cdot \Expectation{\X \sim M}{ \left| \etaCont( \X ) - \frac{1}{2} \right| \cdot \indicator{ (2\etaCont( \X ) - 1) \cdot s( \X ) < 0 } } \\
    &= \frac{1}{1 - 2 \cdot \rho_{\mathrm{max}}} \cdot \reg( s; \DCont, \ellZO ),
\end{align*}
where the penultimate line is because $\eta( x ) > \frac{1}{2} \iff \etaCont( x ) > \frac{1}{2}$ by Corollary \ref{corr:noisy-clean-01}.
\end{proof}

\section{Simplified proofs of Proposition \ref{prop:eta-monotone}}
\label{app:simplified-monotone}

We present some simplified proofs of Proposition \ref{prop:eta-monotone} in some special cases.
Appendix \ref{app:gpcn-monotone} considers the case of the symmetric PTN model.
Appendix \ref{app:eta-monotone-diffble-proof} considers the case when $f_{\pm 1}$ are differentiable.

\subsection{Proof for PTN model}
\label{app:gpcn-monotone}

We now show that the PTN model of Example \ref{ex:gpcn} will guarantee $\etaCont$ is order preserving for $\eta$.


Suppose $\DCont = \PTN( \D, g, g )$ for some $g \colon [0, 1] \to [0, 1/2)$.
Recall from Equation \ref{eqn:eta-ptn} that
\begin{align*}
    \etaCont( x ) &= \varphi( \eta( x ) )
\end{align*}
where $\varphi( z ) = (1 - 2 \cdot g( z )) \cdot z + g( z )$.
Therefore, we just need to establish strict monotonicity of $\varphi$.

For differentiable $g$, strict monotonicity is easy to establish: this is because
\begin{align*}
    \varphi'( z ) &= 1 - 2 \cdot g( z ) - 2 \cdot g'( z ) \cdot z + g'( z ) \\
    &= 1 - 2 \cdot g( z ) + g'( z ) \cdot (1 - 2 \cdot z ).
\end{align*}
Since $1 - 2 \cdot g( z ) \geq 1 - 2 \cdot \rho_{\mathrm{max}}$, and $g'( z ) \geq 0 \iff z \leq \frac{1}{2}$ by definition of the PTN model (see Example \ref{ex:gpcn}), we have
$\varphi'( z ) \geq 1 - 2 \cdot \rho_{\mathrm{max}} > 0$.

For non-differentiable $g$, we must explicitly check that $x < y \implies \varphi( x ) < \varphi( y )$.
We have
\begin{align*}
    \varphi( x ) - \varphi( y ) &= x - y + g( x ) \cdot (1  - 2 \cdot x) - g( y ) \cdot (1 - 2 \cdot y) \\
    &= x - y - g( x ) \cdot (2 \cdot x - 1) + g( y ) \cdot (2 \cdot y - 1).
\end{align*}
Consider the three possible cases.
\begin{itemize}
    \item Suppose $x \leq \frac{1}{2} \leq y$.
Then, $1 - 2 \cdot x \geq 0$ and $2 \cdot y - 1 \geq 0$, and so
\begin{align*}
    \varphi( x ) - \varphi( y ) &= x - y + g( x ) \cdot (1 - 2 \cdot x) + g( y ) \cdot (2 \cdot y - 1) \\
    &\leq x - y + 2 \cdot \max( g(x), g(y) ) \cdot (y - x) \\
    &= (x - y) \cdot (1 - 2 \cdot \max( g(x), g(y) ) ) \\
    &< 0,
\end{align*}
since $x - y < 0$ and $1 - 2 \cdot \max( g(x), g(y) ) > 0$.

    \item Suppose $\frac{1}{2} \leq x < y$.
Then, since $g$ is decreasing on $[\frac{1}{2}, 1]$, $g( x ) > g( y )$, and so
\begin{align*}
    \varphi( x ) - \varphi( y ) &< x - y - 2 \cdot g( x ) \cdot (x - y) \\
    &= (x - y) \cdot (1 - 2 \cdot g( x )) \\
    &< 0,
\end{align*}
since $x - y < 0$ and $1 - 2 \cdot g( x ) > 0$.

    \item Suppose $x < y \leq \frac{1}{2}$. Then, $g( x ) < g( y )$ and so
\begin{align*}
    \varphi( x ) - \varphi( y ) &< x - y + 2 \cdot g( y ) \cdot ( y - x ) \\
    &= (x - y) \cdot (1 - 2 \cdot g( y )) \\
    &< 0,
\end{align*}
since $x - y < 0$ and $1 - 2 \cdot g( y ) > 0$.
\end{itemize}

Thus, we conclude that $\eta( x ) < \eta( x' ) \implies \varphi( \eta( x ) ) < \varphi( \eta( x' ) ) \implies \etaCont( x ) < \etaCont( x' )$.

%
\subsection{Proof for differentiable $f_{\pm 1}$}
\label{app:eta-monotone-diffble-proof}

For the case of differentiable $f_{\pm 1}$, the following is one consequence of Assumption \ref{ass:c}.

\begin{lemma}
\label{lemm:ass-c}
Pick any $\D$.
Suppose $f_{\pm 1}$ are differentiable, and $(f_{-1}, f_{1}, s, \eta)$
are BCN-admissible (Equation \ref{eqn:bcn-admissible}),
and additionally satisfy Assumption \ref{ass:c}.
Then,
$$ ( f_{-1}'( z ) + f_{1}'( z ) ) \cdot u( z ) \leq f_{-1}'( z ) $$
where $\eta = u \circ s$.
\end{lemma}

\begin{proof}
Observe that
\begin{align*}
    ( f_{-1}'( z ) + f_{1}'( z ) ) \cdot u( z ) - f_{-1}'( z ) &= f_{-1}'( z ) \cdot (u( z ) - 1) + f_{1}'( z ) ) \cdot u( z ) \\
    &= f_{-1}'( z ) \cdot (u( z ) - \nicefrac[]{1}{2}) + f_{1}'( z ) ) \cdot (u( z ) - \nicefrac[]{1}{2}) + \frac{1}{2} \cdot (f_{1}'( z ) + f_{-1}'( z )) \\
    &= (f_{-1}'( z ) + f_{1}'( z )) \cdot (u( z ) - \nicefrac[]{1}{2}) + \frac{1}{2} \cdot (f_{1}'( z ) - f_{1}'( z )).
\end{align*}
The first term is $\leq 0$ by Condition (b) of BCN-admissibility.
The second term is $\leq 0$ by Assumption \ref{ass:c}.
Thus, the result is shown.

\end{proof}

We use this to show the desired order preserving property of $\etaCont$.
Recall from Equation \ref{eqn:eta-bcn} that for a BCN model,
$$ \etaCont( x ) = \varphi( s( x ) ) $$
where
$$ \varphi( z ) = (1 - f_1( z ) - f_{-1}( z )) \cdot u( z ) + f_{-1}( z ). $$
Assuming all terms are differentiable, we have
\begin{align*}
    \varphi'( z ) &= (1 - f_1( z ) - f_{-1}( z )) \cdot u'( z ) + f'_{-1}( z ) - (f'_{1}( z ) + f'_{-1}( z )) \cdot u( z ).
\end{align*}
Since $1 - f_1( z ) - f_{-1}( z ) > 0$ by Assumption \ref{ass:total-noise}, $u'( z ) \geq 0$ by monotonicity of $u$, and the last term is $\geq 0$ by Lemma \ref{lemm:ass-c}, we have $\varphi'( z ) \geq 0$.
Further, $\varphi'( z ) = 0$ only if $u'( z ) = 0$, meaning $\varphi$ is strictly monotone whenever $u$ is, \ie
$$ ( \forall x, y \in \Real ) \, u( x ) < u( y ) \implies \varphi( x ) < \varphi( y ), $$
which in turn means that
$$ ( \forall x, x' \in \XCal ) \, \eta( x ) < \eta( x' ) \implies \etaCont( x ) < \etaCont( x' ). $$

\section{Examples of corrupted SIM members}
\label{app:corrupt-sim-examples}

We present two examples of SIM members corrupted by noise following the SIN model.

\begin{example}
Suppose we are in the CCN regime, so that $f_1 \equiv \alpha, f_{-1} \equiv \beta$ for admissible $\alpha, \beta < 1$.
Then, as per Equation \ref{eqn:ccn-eta},
$$ \Contaminator{u}( z ) = ( 1 - \alpha - \beta ) \cdot u( z ) + \beta. $$
That is, the corrupted class-probability function is a scaled and translated version of the original class-probability function.
If further $ u( z ) = \indicator{ z > 0 } $, so that $\D$ is separable, we have
$$ \Contaminator{u}( z ) = \begin{cases} 1 - \alpha & \text{ if } z > 0 \\ \beta & \text{ if } z < 0. \end{cases} $$
That is, the corrupted class-probability function takes on two unique values, depending on which side of the optimal hyperplane one is on.
\end{example}

\begin{example}
\label{ex:bylander-separable}
Suppose we are in the Bylander regime, so that $f_1 \equiv f_{-1} \equiv f$ and $f( z ) = g( | z | )$ for some arbitrary monotone decreasing function $g$.
Then,
$$ \Contaminator{u}( z ) = ( 1 - 2 \cdot f( z) ) \cdot u( z ) + f( z ). $$

If further assume $ u( z ) = \indicator{ z > 0 } $, so that $\D$ is separable, we have
\begin{align*}
     \Contaminator{u}( z ) &= \begin{cases} 1 - f( z ) & \text{ if } z > 0 \\ f( z ) & \text{ if } z < 0 \end{cases} \\
     &= \begin{cases} 1 - g( z ) & \text{ if } z > 0 \\ g( -z ) & \text{ if } z < 0. \end{cases}
 \end{align*}
Observe that if $g$ satisfies $g( -z ) = 1 - g( z )$, then this is
$$ \Contaminator{u}( z ) = g( -z ). $$
That is, a structured form of monotonic noise on a linearly separable distribution yields a distribution scorable by some generalised linear model.
In the case where $g( z ) = {1}/({1 + e^{z}})$ for example, we end up with a logistic regression model.
This observation has been made previously, \eg \citet{Du:2015}.
\end{example}




\section{Application of Isotron to CCN setting}
\label{app:isotron-ccn}

To see the challenge in estimating $\etaCont$, recall the following example.

\begin{example} 
Suppose that $ \eta( x ) = u( \langle w^*, x \rangle )$ for some known $u$.
Suppose that $f_1 \equiv \rho_1, f_{-1} \equiv \rho_{-1}$ for admissible $\rho_1, \rho_{-1} < 1$.
Then,
$$ \etaCont( x ) = ( 1 - \rho_1 - \rho_{-1} ) \cdot u( \langle w^*, x \rangle ) + \rho_{-1}, $$
or for simplicity
$$ \etaCont( x ) = \alpha \cdot u( \langle w^*, x \rangle ) + \beta. $$
\end{example}

If we had access to clean samples, then we could minimise the canonical loss corresponding to the link function $u$ over the class of linear scorers (a simple convex objective) in order to recover $w^*$ asymptotically.
For example, if we know $u$ is a sigmoid, we would minimise the logistic loss on clean samples.

Can we similarly learn $\etaCont$ from corrupted samples?
If we knew the parameters $\alpha, \beta$, then the same procedure could be applied.
However, we unfortunately do not know these in general, and must be estimated as well.
Estimating $\alpha, \beta$ means that, effectively, we are also estimating the link function.
We thus are apparently faced with the challenging problem of having to \emph{learn the link function as well as the weight}.
(Note that the resulting problem is entirely equivalent to learning a neural network with a single hidden unit.)


To solve this problem,
one might resort to an alternating procedure wherein one alternately takes a gradient step in the direction of $w$, and then in $\alpha, \beta$.
This approach is simple, but it is unclear whether the procedure is consistent.
Thus, the Isotron solves a non-trivial estimation problem.

\end{document}